\documentclass[journal]{IEEEtran}
\usepackage{tabularx}
\usepackage{cite}
\usepackage{algorithmic}
\usepackage{algorithm}
\usepackage{multirow}
\usepackage{graphicx,amsmath, amssymb,lineno}
\usepackage{subfigure, graphicx}
\usepackage{amsmath,amssymb,amsthm}
\usepackage{afterpage}

\newtheorem{theorem}{Theorem}
\newtheorem{prop}{Proposition}

\newtheorem{lemma}{Lemma}
\newtheorem{remark}{Remark}

\begin{document}

\title{ Effective Discriminative Feature Selection with  Non-trivial Solutions}

\author{Hong~Tao, Chenping~Hou$^*$,~\IEEEmembership{Member,~IEEE}, Feiping~Nie$^*$, Yuanyuan Jiao, Dongyun~Yi
\thanks{This work was supported by NSF China (No. 61473302).}
\thanks{Hong Tao, Chenping Hou and Dongyun Yi are with the College of Science, National University of Defense Technology, Changsha, Hunan, 410073, China (E-mail: taohong.nudt@gmail.com; hcpnudt@hotmail.com; dongyun.yi@gmail.com).}
\thanks{Feiping Nie is with the Center for OPTical IMagery Analysis and Learning (OPTIMAL), Northwestern Polytechnical University, Xi'an, Shanxi, 710072, China. (E-mail: feipingnie@gmail.com).}
\thanks{Yuanyuan Jiao is with the College of Nine, National University of Defense Technology, Changsha, Hunan, 410073, China (E-mail: jyynudt@gmail.com).}
\thanks{$^*$Corresponding authors.}
}

\date{}

\maketitle

\begin{abstract}
Feature selection and feature transformation, the two main ways to reduce dimensionality, are often presented separately. In this paper, a feature selection method is proposed by combining the popular transformation based dimensionality reduction method Linear Discriminant Analysis (LDA) and sparsity regularization. We impose row sparsity on the transformation matrix of LDA through ${\ell}_{2,1}$-norm regularization to achieve feature selection, and the resultant formulation optimizes for selecting the most discriminative features and removing the redundant ones simultaneously. The formulation is extended to the ${\ell}_{2,p}$-norm regularized case: which is more likely to offer better sparsity when $0<p<1$. Thus the formulation is a better approximation to the feature selection problem.
An efficient algorithm is developed to solve the ${\ell}_{2,p}$-norm based optimization problem and it is proved that the algorithm converges when $0<p\le 2$. Systematical experiments are conducted to understand the work of the proposed method. Promising experimental results on various types of real-world data sets demonstrate the effectiveness of our algorithm.
\end{abstract}


\section{Introduction}

In many applications in computer vision, data mining and pattern recognition, data are characterized by tens or hundreds of thousands of variables or features. High dimensionality significantly increases the time and space requirements for processing the data. Moreover, some features are irrelevant and  redundant. The existence of these features may result in low efficiency, over-fitting and poor prediction performance in learning tasks \cite{FSforKDD,Lipo_WrapperFS,Nie_DLSR,Guyon03,Liu_GLSP}. Consequently, dimensionality reduction has become an important stage of data preprocessing in such applications \cite{SDR, MVSSDR}.

Feature selection and feature transformation are the two main ways to reduce dimensionality \cite{Chen_DR,LingShao_FelLearn}. Whereas feature
transformation methods transform the original features to a new feature subspace, feature selection
selects a subset of features of the original set. In contrast to feature transformation, feature selection does not alter the original representation of the variables. Thus, feature selection preserves the original semantics of the variables, thereby offering the advantage of interpretability. Another advantage of feature selection is that only the selected features need to be collected or calculated, while all input features are required to obtain the low-dimensional representation in feature transformation methods.
As a result, many studies focus on addressing the problem of feature selection during the past a few years.

While feature selection can be applied to both supervised and unsupervised learning, we focus on the problem of supervised learning (classification), where the label information is available. According to how the classification algorithm is incorporated in evaluating and selecting features, feature selection methods can be organized into three categories \cite{Guyon03}: (1) filter methods \cite{Bishop_fisher,ReliefF,Peng_mRMR,LapalcianScore}, where the selection is independent of the classifiers, (2) wrapper methods \cite{Lipo_WrapperFS,Guyon03}, where a feature subset search algorithm is wrapped around the classification model and the feature subsets are scored based on their predictive power, and (3) embedded methods \cite{L1SVM,Weston_L0norm}, which search for an optimal subset of features in the process of classifier construction. Compared to filter methods, wrapper methods and embedded methods are tightly coupled with a specific classifier, thus they often have good performance but also very expensive computational costs.
In this paper, we focus on the filter-type methods for supervised feature selection.

Filter-based feature selection methods can be classified into two subtypes: (1) feature ranking (univariate techniques) and (2) feature subset evaluation (multivariate techniques). Filter-based feature selection utilizes the intrinsic properties of the data to evaluate the importance of (1) each individual feature in feature ranking methods or (2) the entire feature subset in the case of feature subset evaluation, with respect to (w.r.t) a certain proposed performance criterion.
Feature ranking methods often find suboptimal solutions due to the following two reasons: (1) The interaction among features is neglected. Feature interaction exists if a feature forms a subset with other ones and the subset has strong correlation with the class \cite{Rel_Definition,ASUfspackage}. Evaluating features individually does not consider the relevance of a feature subset, and features in a relevant subset will be removed if they are low-scored. (2) Redundant features, i.e., features with similar predictive power, or more specifically, highly correlated features, cannot be eliminated if they are all highly scored. In fact, many studies \cite{Peng_mRMR,Rel_Definition} have shown that removing redundant features can improve the prediction accuracy.
Multivariate techniques overcome this problem to some degree.
Nevertheless, both subtypes of filter-based feature selection methods use only the intrinsic characteristics of the data without using the learning mechanism. This mechanism is proved to be powerful and has been widely used in many areas \cite{ASUfspackage,FSMRank,JELSR}.

The goal of supervised feature selection is to find the most discriminative features that can distinguish different classes. Thus discriminant analysis plays an important role in supervised feature selection\cite{Optic_DisFS,Semi_DFS,Nie_DLSR}.
The Fisher Score algorithm\cite{Bishop_fisher} is a widely applied filter-type feature selection algorithm based on linear discriminant analysis (LDA). However, it values features individually and therefore cannot deal with feature interaction and feature redundancy, i.e., the high correlation among the features \cite{ASUfspackage}.
S. Niijima and S. Kuhara use the Maximum Margin Criterion (MMC), a variant of LDA, for feature selection.
They recursively remove features with the smallest absolute values of the discriminant vectors yielded by MMC, until the desired number of features are removed\cite{MMC-RFE}.

M. Masaeli {\em et al}. propose converting LDA into a new filter-based feature selection algorithm named Linear Discriminant Feature Selection (LDFS) \cite{icml10_LDFS}. By enforcing row sparsity on the transformation matrix of LDA through ${{\ell}_{\infty,1}}$-norm regularization, LDFS uses both the discriminative information and the learning mechanism. As the features are selected jointly by a learning process, LDFS manages to optimize for feature relevance and redundancy removal simultaneously.
However, the formulation of LDFS ignores the possibility of arbitrary scalability of the transformation matrix, so that it has a trivial solution of all zeros. Thus, it would lose its ability to select features when arriving at the trivial solution.

In this paper, we first prove the existence of the trivial solution of the formulation of LDFS.
Then a new formulation is propounded to avoid the trivial solution, in which the transformation vectors are constrained to be uncorrelated. Instead of utilizing ${{\ell }_{\infty ,1}}$-norm regularization, we adopt ${{\ell }_{2,1}}$-norm minimization, which can be solved by a simpler algorithm and ensures the ability of feature selection as well.
The proposed formulation not only avoids the trivial solution, but also inherits LDFS's merit of selecting the most discriminative features and removing the redundant ones simultaneously.

Both ${{\ell }_{\infty ,1}}$-norm and ${{\ell }_{2,1}}$-norm are extensions of ${{\ell }_{1}}$-norm. ${{\ell }_{1}}$-norm is used most frequently to find sparse solutions for its convexity. In fact, using ${{\ell }_{p}}$-norm ($0<p<1$) can find sparser solutions than using ${{\ell }_{1}}$-norm \cite{Signal_nonconvex,Chartrand_CS2008,mixedL2p}, but it is challenging to solve the corresponding non-convex optimization problem.
In this paper, we manage to generalize our formulation to the non-convex ${{\ell }_{2,p}}$-norm regularization case, which is expected to have better sparsity than ${{\ell }_{2,1}}$-norm minimization \cite{mixedL2p}. We develop a simple algorithm to solve our proposed Discriminative Feature Selection (DFS).
The convergence of the algorithm is rigorously proved for $p$ in $\left( 0,2 \right]$ which covers the range we are interested in. Our contributions are summarized as,

\begin{itemize}
  \item Prove the formulation of LDFS has a trivial solution of all zeros;
  \item Propose a new formulation to avoid the trivial solution based on ${{\ell }_{2,1}}$-norm regularization and extend it to the ${{\ell }_{2,p}}$-norm regularized cases. When $0<p<1$, the ${{\ell }_{2,p}}$-norm regularization is likely to offer better sparsity, thus the formulation is more suitable for feature selection;
  \item Develop an efficient algorithm to address the ${{\ell }_{2,p}}$-norm regularized optimization problem and rigorously proving that the algorithm monotonically decreases the objective of DFS with $0<p\le 2$;
  \item Evaluate DFS systematically on various types of real-world data sets to understand the ability of DFS to select discriminative features and remove redundant features.
\end{itemize}

The rest of the paper is organized as follows. Section II states some necessary notations and definitions. A brief review of the LDFS approach is given in Section III. In Section IV, we will introduce the formulation of DFS and provide an efficient solution algorithm. Section V presents deep analysis of the proposed method, including convergence  behavior, time complexity etc. Experimental results on various kind of data sets are displayed in Section VI.
The conclusion and the future work are in Section VII.

\section{Notations and Definitions}
We introduce the notations and the definitions of norms used in this paper. Matrices and vectors are written as boldface uppercase letters and boldface lowercase letters respectively. For an matrix $\mathbf{M}=({{m}_{ij}})$, its $i$-th row, $j$-th column are denoted by ${{\mathbf{m}}^{i}}$, ${{\mathbf{m}}_{j}}$ respectively.
The ${{\ell }_{p}}$-norm ($p>0$) of a vector $\mathbf{v}\in {{\mathbb{R}}^{n}}$ is defined as ${{\left\| \mathbf{v} \right\|}_{p}}={{\left( \sum\limits_{i=1}^{n}{{{\left| {{v}_{i}} \right|}^{p}}} \right)}^{\frac{1}{p}}}$, and the ${{\ell }_{0}}$-norm is defined as ${{\left\| \mathbf{v} \right\|}_{0}}=\sum\limits_{i=1}^{n}{{{\left| {{v}_{i}} \right|}^{0}}}$. Actually, neither ${{\ell }_{0}}$ nor ${{\ell }_{p}}$ $(0<p<1)$ is a valid norm, because the former does not satisfy the positive scalability: ${{\left\| \alpha \mathbf{v} \right\|}_{0}}=\left| \alpha  \right|{{\left\| \mathbf{v} \right\|}_{0}}$ for scalar $\alpha $ and the latter does not satisfy the triangular inequality (${{\left\| \mathbf{u}+\mathbf{v} \right\|}_{p}}\nleqslant{{\left\| \mathbf{u} \right\|}_{p}}+{{\left\| \mathbf{v} \right\|}_{p}}$, $0<p<1$). We call them norms here for convenience.

The ${{\ell }_{2,1}}$-norm of an matrix $\mathbf{M}\in {{\mathbb{R}}^{n\times m}}$ is defined as \cite{Nie_RFSL21}
\begin{equation}
{{\left\| \mathbf{M} \right\|}_{2,1}}=\sum\limits_{i=1}^{n}{\sqrt{\sum\limits_{j=1}^{m}{m_{ij}^{2}}}}=\sum\limits_{i=1}^{n}{{{\left\| {{\mathbf{m}}^{i}} \right\|}_{2}}}.
\end{equation}
The ${{\ell }_{2,1}}$-norm can be generalized to ${{\ell }_{r,p}}$-norm
\begin{equation}
\begin{split}
{{\left\| \mathbf{M} \right\|}_{r,p}} & ={{\left( {{\sum\limits_{i=1}^{n}{\left( \sum\limits_{j=1}^{m}{{{\left| {{m}_{ij}} \right|}^{r}}} \right)}}^{\frac{p}{r}}} \right)}^{\frac{1}{p}}} \\
& ={{\left( \sum\limits_{i=1}^{n}{\left\| {{\mathbf{m}}^{i}} \right\|_{r}^{p}} \right)}^{\frac{1}{p}}},r>0,p>0,
\end{split}
\end{equation}

\begin{equation}
{{\left\| \mathbf{M} \right\|}_{r,0}}={{\sum\limits_{i=1}^{n}{\left( \sum\limits_{j=1}^{m}{{{\left| {{m}_{ij}} \right|}^{r}}} \right)}}^{0}}=\sum\limits_{i=1}^{n}{\left\| {{\mathbf{m}}^{i}} \right\|_{r}^{0}},r>0.
\end{equation}
Under this definition, the ${{\ell }_{r,0}}$-norm of a matrix $\mathbf{M}$ is exactly the number of nonzero rows of $\mathbf{M}$.

When $r\ge 1$ and $p\ge 1$, ${{\ell }_{r,p}}$-norm is a valid norm as it satisfies the three norm conditions, including the triangle inequality ${{\left\| \mathbf{A} \right\|}_{r,p}}+{{\left\| \mathbf{B} \right\|}_{r,p}}\ge {{\left\| \mathbf{A}+\mathbf{B} \right\|}_{r,p}}$. This can be simply proved as follows. Using the triangle inequality of ${{\ell }_{p}}$-norm:${{\left( \sum\nolimits_{i}{{{\left| {{u}_{i}} \right|}^{p}}} \right)}^{\frac{1}{p}}}+{{\left( \sum\nolimits_{i}{{{\left| {{v}_{i}} \right|}^{p}}} \right)}^{\frac{1}{p}}}\ge {{\left( \sum\nolimits_{i}{{{\left| {{u}_{i}}+{{v}_{i}} \right|}^{p}}} \right)}^{\frac{1}{p}}}$ ($p\ge 1$) and setting ${{u}_{i}}={{\left\| {{\mathbf{a}}^{i}} \right\|}_{r}}$ and ${{v}_{i}}={{\left\| {{\mathbf{b}}^{i}} \right\|}_{r}}$, then we obtain
\begin{equation}\label{formula1}
\begin{split}
&{{\left( \sum\limits_{i}{\left\| {{\mathbf{a}}^{i}} \right\|_{r}^{p}} \right)}^{\frac{1}{p}}}+{{\left( \sum\limits_{i}{\left\| {{\mathbf{b}}^{i}} \right\|_{r}^{p}} \right)}^{\frac{1}{p}}} \\
&\ge {{\left( {{\sum\limits_{i}{\left| {{\left\| {{\mathbf{a}}^{i}} \right\|}_{r}}+{{\left\| {{\mathbf{b}}^{i}} \right\|}_{r}} \right|}}^{p}} \right)}^{\frac{1}{p}}} \\
&\ge {{\left( {{\sum\limits_{i}{\left| {{\left\| {{\mathbf{a}}^{i}}+{{\mathbf{b}}^{i}} \right\|}_{r}} \right|}}^{p}} \right)}^{\frac{1}{p}}}, r\ge 1, p\ge 1,
\end{split}
\end{equation}
where the second inequality follows the triangle inequality for norms: ${{\left\| {{\mathbf{a}}^{i}} \right\|}_{r}}+{{\left\| {{\mathbf{b}}^{i}} \right\|}_{r}}\ge {{\left\| {{\mathbf{a}}^{i}}+{{\mathbf{b}}^{i}} \right\|}_{r}}$.
(\ref{formula1}) is just ${{\left\| \mathbf{A} \right\|}_{r,p}}+{{\left\| \mathbf{B} \right\|}_{r,p}}\ge {{\left\| \mathbf{A}+\mathbf{B} \right\|}_{r,p}}$.
However, when $0<r<1$ or $0\le p<1$, ${{\ell }_{r,p}}$ is not a valid matrix norm. Still, here we call them norms for convenience. The notations used in this paper are summarized in Table \ref{notation}.

\begin{table}
  \centering
  \caption{Notations}\label{notation}
  \begin{tabular}{|l | l |}\hline
  Notations         & Descriptions                   \\ \hline
  $d$               & The dimensionality of the original data \\
  $n$               & The data size                        \\
  $c$               & The number of classes                 \\
  ${n}_{k}$         & The number of data points in the $k$-th class \\
  $l$               & The reduced dimensionality             \\
  $\mathcal{F}$     & The set of selected features        \\
  ${{\mathbf{x}}_{i}}\in {{\mathbb{R}}^{d}}$ & The $i$-th data point \\
  ${{\mathbf{x}}_{i}^{(k)}}\in {{\mathbb{R}}^{d}}$ & The $i$-th data point in the $k$-th class \\
  ${\mathbf{X}}\in {{\mathbb{R}}^{n\times d}}$ & The data matrix \\
  $\boldsymbol{\mu }\in {{\mathbb{R}}^{d}}$ & The total sample mean vector \\
  ${{\boldsymbol{\mu }}^{(k)}}\in {{\mathbb{R}}^{d}}$ & The mean vector of the $k$-th class \\
  ${\mathbf{f}_{i}}\in {{\mathbb{R}}^{n}}$  & The samples for the $i$-th feature                        \\
  $\mathbf{a}\in {\mathbb{R}}^{d}$ & The transformation vector \\
  $\mathbf{A}\in {{\mathbb{R}}^{d\times l}}$      & The transformation matrix \\
  ${\mathbf{S}}_{t} \in {{\mathbb{R}}^{d\times d}}$  & The total scatter matrix \\
  ${\mathbf{S}}_{w} \in {{\mathbb{R}}^{d\times d}}$  & The within-class scatter matrix \\
  ${\mathbf{S}}_{b} \in {{\mathbb{R}}^{d\times d}}$  & The between-class scatter matrix \\  \hline
  \end{tabular}
\end{table}

\section{Linear Discriminant Feature Selection Revisited}
In this section, after a brief review of LDA, we introduce the feature selection method LDFS, which is derived from LDA. Then we prove that there is a trivial solution of the formulation of LDFS.

LDA is a popular supervised transformation-based dimensionality reduction method. It seeks directions on which the data points of different classes are far from each other, while data points in the same class are close to each other. Suppose we have a set of $n$ samples ${{\mathbf{X}}}={[{{\mathbf{x}}_{1}},{{\mathbf{x}}_{2}},\cdots ,{{\mathbf{x}}_{n}}]}^{T}\in {{\mathbb{R}}^{n\times d}}$, belonging to $c$ classes.
The objective function of LDA is as follows \cite{IntrStaPR}:
\begin{equation}\label{LDA1}
{{\mathbf{a}}^{*}}=\underset{\mathbf{a}}{\mathop{\arg \max }}\,
\frac{{{\mathbf{a}}^{T}}{{\mathbf{S}}_{b}}\mathbf{a}}{{{\mathbf{a}}^{T}}{{\mathbf{S}}_{w}}\mathbf{a}},
\end{equation}
\begin{equation}
{{\mathbf{S}}_{b}}=\sum\limits_{k=1}^{c}{{{n}_{k}}({{\boldsymbol{\mu}}^{(k)}}-\boldsymbol{\mu }){{({{\boldsymbol{\mu }}^{(k)}}-\boldsymbol{\mu })}^{T}}},
\end{equation}
\begin{equation}
{{\mathbf{S}}_{w}}=\sum\limits_{k=1}^{c}{\left( \sum\limits_{i=1}^{{{n}_{k}}}{(\mathbf{x}_{i}^{(k)}-{{\boldsymbol{\mu }}^{(k)}}){{(\mathbf{x}_{i}^{(k)}-{{\boldsymbol{\mu }}^{(k)}})}^{T}}} \right)},
\end{equation}
where $\boldsymbol{\mu }$ is the total sample mean vector, ${{n}_{k}}$ is the number of samples in the $k$-th class, ${{\boldsymbol{\mu }}^{(k)}}$ is the average vector of the $k$-th class, and $\mathbf{x}_{i}^{(k)}$ is the $i$-th sample in the $k$-th class. We call ${{\mathbf{S}}_{w}}$ the within-class scatter matrix and ${{\mathbf{S}}_{b}}$ the between-class scatter matrix.

Define ${{\mathbf{S}}_{t}}=\sum\nolimits_{i=1}^{n}{({{\mathbf{x}}_{i}}-\boldsymbol{\mu }){{({{\mathbf{x}}_{i}}-\boldsymbol{\mu })}^{T}}}$ as the total scatter matrix, then we have ${{\mathbf{S}}_{t}}={{\mathbf{S}}_{b}}+{{\mathbf{S}}_{w}}$.
The objective function of LDA in (\ref{LDA1}) is equivalent to \cite{IntrStaPR}
\begin{equation}\label{LDA2}
{{\mathbf{a}}^{*}}=\underset{\mathbf{a}}{\mathop{\arg \max }}\,\frac{{{\mathbf{a}}^{T}}{{\mathbf{S}}_{b}}\mathbf{a}}{{{\mathbf{a}}^{T}}{{\mathbf{S}}_{t}}\mathbf{a}}.
\end{equation}
When $l$ projective functions $\mathbf{A}=[{{\mathbf{a}}_{1}},{{\mathbf{a}}_{2}},\cdots ,{{\mathbf{a}}_{l}}]$ are needed, the objective function of LDA can be written as
\begin{equation}\label{LDA3}
{{\mathbf{A}}^{*}}=\underset{\mathbf{A}\in {{\mathbb{R}}^{d\times l}}}{\mathop{\arg \max }}\,tr({{({{\mathbf{A}}^{T}}{{\mathbf{S}}_{t}}\mathbf{A})}^{-1}}({{\mathbf{A}}^{T}}{{\mathbf{S}}_{b}}\mathbf{A})),
\end{equation}
or
\begin{equation}\label{LDA4}
{{\mathbf{A}}^{*}}=\underset{\mathbf{A}\in {{\mathbb{R}}^{d\times l}}}{\mathop{\arg \min }}\,-tr({{({{\mathbf{A}}^{T}}{{\mathbf{S}}_{t}}\mathbf{A})}^{-1}}({{\mathbf{A}}^{T}}{{\mathbf{S}}_{b}}\mathbf{A})).
\end{equation}

Before introducing the formulation of LDFS, we first analyze how the structure of $\mathbf{A}$ should be to achieve feature selection. To preserve the semantic consistency, here we denote the data's $j$-th feature, i.e., the $j$-th column of $\mathbf{X}$, as ${{\mathbf{f}}_{j}}$. If all the elements of the $j$-th row of $\mathbf{A}$ are zero, then feature ${{\mathbf{f}}_{j}}$ makes no contribution to the low-dimensional data representation $\mathbf{XA}$ and it should be removed by the feature selection methods. On the other hand, if feature ${{\mathbf{f}}_{j}}$ is selected by the feature selection algorithm, then there is at least one element of the $j$-th row of $\mathbf{A}$ to be nonzero.
Hence, forcing the transformation matrix $\mathbf{A}$ to have more zero rows can be interpreted as selecting fewer features.  Using this idea, M. Masaeli {\em et al}. convert LDA into a feature selection algorithm, LDFS, through ${{\ell }_{\infty ,1}}$-norm regularization\footnote{The first term $-\frac{{{\mathbf{A}}^{T}}{{\mathbf{S}}_{b}}\mathbf{A}}{{{\mathbf{A}}^{T}}{{\mathbf{S}}_{w}}\mathbf{A}}$  in this formulation should be converted into a number, however, in \cite{icml10_LDFS} the authors do not specify which criterion of LDA is used. Thus, we just keep the original formulation in \cite{icml10_LDFS}.  } \cite{icml10_LDFS} :
\begin{equation}\label{LDFS}
\underset{\mathbf{A}\in {{\mathbb{R}}^{d\times l}}}{\mathop{\min }}\,-\frac{{{\mathbf{A}}^{T}}{{\mathbf{S}}_{b}}\mathbf{A}}{{{\mathbf{A}}^{T}}{{\mathbf{S}}_{w}}\mathbf{A}}+\gamma {{\left\| \mathbf{A} \right\|}_{\infty ,1}}=-\frac{{{\mathbf{A}}^{T}}{{\mathbf{S}}_{b}}\mathbf{A}}{{{\mathbf{A}}^{T}}{{\mathbf{S}}_{w}}\mathbf{A}}+\gamma \sum\limits_{i=1}^{d}{{{\left\| {{\mathbf{a}}^{i}} \right\|}_{\infty }}},
\end{equation}
where $\gamma >0$ is the parameter that tunes the row sparsity of the transformation matrix $\mathbf{A}$. Increasing $\gamma$ means forcing more rows to be zero, thus more features will be removed.
The ${{\ell }_{\infty }}$-norm of the vector ${{\mathbf{a}}^{i}}$ is the maximum of the absolute value of the elements of ${{\mathbf{a}}^{i}}$ and the ${{\ell }_{1 }} $-norm induces sparsity. As a result, the formulation of LDFS imposes sparsity on the maximum absolute value of the elements of each row of $\mathbf{A}$, thereby pushing all the elements of each row to zero.

To optimize for the ${{\ell }_{\infty }}$-norm, M. Masaeli {\em et al}. \cite{icml10_LDFS} adopt a vector of dummy variables to represent the maximum absolute value of the elements of rows of the transformation matrix, transforming the formulation into an optimization problem with box constraints.  A Quasi-Newton method is used to solve this problem, but the evaluation of the gradient of $-\frac{{{\mathbf{A}}^{T}}{{\mathbf{S}}_{b}}\mathbf{A}}{{{\mathbf{A}}^{T}}{{\mathbf{S}}_{w}}\mathbf{A}}$ is computationally very expensive.
Moreover, (\ref{LDFS}) has a trivial solution of all zeros. We prove this in Proposition \ref{prop1}.

\begin{prop}\label{prop1}
The formulation of LDFS defined in (\ref{LDFS}) has a trivial solution of all zeros.
\end{prop}

\begin{proof}
Let $\mathcal{J}(\mathbf{XA})=-\frac{{{\mathbf{A}}^{T}}{{\mathbf{S}}_{b}}\mathbf{A}}{{{\mathbf{A}}^{T}}{{\mathbf{S}}_{w}}\mathbf{A}}+\gamma {{\left\| \mathbf{A} \right\|}_{\infty ,1}}$ and suppose ${{\mathbf{A}}^{*}}$ is a solution of (\ref{LDFS}), then $c{{\mathbf{A}}^{*}}$ is a better solution of (\ref{LDFS}), where $c$ is a nonzero constant with $\left| c \right|<1$ :
\begin{equation}
\begin{split}
\mathcal{J}(\mathbf{X}(c{{\mathbf{A}}^{*}})) & =-\frac{{{c}^{2}}{{\mathbf{A}}^{*}}^{T}{{\mathbf{S}}_{b}}{{\mathbf{A}}^{*}}}{{{c}^{2}}{{\mathbf{A}}^{*}}^{T}{{\mathbf{S}}_{w}}{{\mathbf{A}}^{*}}}+\gamma {{\left\| c{{\mathbf{A}}^{*}} \right\|}_{\infty ,1}} \\
& =-\frac{{{\mathbf{A}}^{*}}^{T}{{\mathbf{S}}_{b}}{{\mathbf{A}}^{*}}}{{{\mathbf{A}^{*}}^{T}}{{\mathbf{S}}_{w}}{{\mathbf{A}}^{*}}}+\gamma \left| c \right|{{\left\| {{\mathbf{A}}^{*}} \right\|}_{\infty ,1}} \\
& \le -\frac{{{\mathbf{A}}^{*}}^{T}{{\mathbf{S}}_{b}}{{\mathbf{A}}^{*}}}{{{\mathbf{A}}^{*}}^{T}{{\mathbf{S}}_{w}}{{\mathbf{A}}^{*}}}+\gamma {{\left\| {{\mathbf{A}}^{*}} \right\|}_{\infty ,1}} \\
& =\mathcal{J}(\mathbf{X}{{\mathbf{A}}^{*}}).
\end{split}
\end{equation}
When $c\to 0$, which means $c{{\mathbf{A}}^{*}}\to \mathbf{0}$, $\mathcal{J}(\mathbf{X}(c{{\mathbf{A}}^{*}}))\to -\frac{{{\mathbf{A}}^{*}}^{T}{{\mathbf{S}}_{b}}{{\mathbf{A}}^{*}}}{{{\mathbf{A}}^{*}}^{T}{{\mathbf{S}}_{w}}{{\mathbf{A}}^{*}}}$, thus $\mathbf{0}$ is the trivial solution of (\ref{LDFS}).
\end{proof}

Recall the desired structure of $\mathbf{A}$ to achieve feature selection, it is not difficult to see that the formulation would lose its ability to select features when arriving at the trivial solution. In the experiments in \cite{icml10_LDFS}, the value of $\gamma$ is increased until the desired number of rows of $\mathbf{A}$, i.e., the number of features to be removed, are close to all-zero (the maximum value of the row is less than 0.01). By this way, the implementation can find a nonzero solution but probably not the optimal solution.

\section{Discriminative Feature Selection Based on ${{\ell }_{2,p}}$-Norm Regularization}

As proved above, the LDFS algorithm proposed in \cite{icml10_LDFS} has a trivial solution of all zeros. It may lose its function of selecting features when it leads to a solution near the trivial solution.
In this section, a new formulation is proposed to avoid the trivial solution. We achieve this by constraining the transformation vectors in LDA to be uncorrelated.
As the optimization of ${\ell}_{\infty,1}$-norm minimization involves extensive computation of evaluating the gradient of $-\frac{{{\mathbf{A}}^{T}}{{\mathbf{S}}_{b}}\mathbf{A}}{{{\mathbf{A}}^{T}}{{\mathbf{S}}_{w}}\mathbf{A}}$, ${\ell}_{2,1}$-norm is adopted instead, and the resulting minimization problem can be solved more easily.
Furthermore, the formulation is generalized to the ${{\ell }_{2,p}}$-norm regularized cases, providing more choices of $p$ values to fit the variety of sparsity requirements.
We develop a very simple algorithm to solve the ${{\ell }_{2,p}}$-norm minimization problem uniformly, and prove it to be convergent when $0<p\le 2$ in next section.
For convenience, we refer our proposed formulation as Discriminative Feature Selection (DFS).

\subsection{Discriminative Feature Selection Based on ${{\ell }_{2,1}}$-Norm Regularization}
According to (\ref{LDA2}) and (\ref{LDA4}) the formulation of LDFS is equivalent to solving the following problem,
\begin{equation}
\underset{\mathbf{A}\in {{\mathbb{R}}^{d\times l}}}{\mathop{\min }}\,-(tr{{({{\mathbf{A}}^{T}}{{\mathbf{S}}_{t}}\mathbf{A})}^{-1}}({{\mathbf{A}}^{T}}{{\mathbf{S}}_{b}}\mathbf{A}))+\gamma {{\left\| \mathbf{A} \right\|}_{\infty ,1}}.
\end{equation}
In general, the regularization term can be set in the form of ${{\ell }_{r,1}}$-norm with $1\le r\le \infty $.
In multi-task feature learning, the choice of $r$ depends on the priori feature sharing between the tasks, from none ($r=1$) to complete ($r=\infty $) \cite{Argyriou_multFL,MultFeaSel}. In fact, a larger $r$ value means allowing better ``group discounts" for sharing the same feature: $r=1$ means linearly growing costs with the number of tasks that use a feature, and $r=\infty$ suggests that only the most demanding task matters \cite{MultFeaSel}.
For single-task learning, increasing $r$ corresponds to more sparsity sharing between the elements in each row of $\mathbf{A}$: from individual element-level sparsity patterns ($r=1$) to row-level sparsity patterns ($r = \infty$).
To perform feature selection, we need to push $\mathbf{A}$ to have zero rows and then remove the corresponding features. Hence, having individual sparsity patterns, i.e., choosing $r=1$, is not suitable for feature selection.
Thus, in order to impose row sparsity on the transformation matrix $\mathbf{A}$ to reach the desired configuration of feature selection, the regularization term can be set in the form of ${{\ell }_{r,1}}$-norm with $1< r\le \infty $.

Here, we adopt ${{\ell}_{2,1}}$-norm regularization as the regularizer for the following two reasons. Firstly, the ${{\ell}_{2,1}}$-norm minimization problem can be solved by a iterative algorithm \cite{Nie_RFSL21}, which is much easier than that of the ${{\ell}_{\infty,1}}$-norm. Secondly, when $p \to 0$, ${\ell}_{2,p}$-norm and ${\ell}_{\infty,p}$-norm have very similar properties, and they both denote the nonzero rows of $\mathbf{A}$ when $p = 0$. So, the reformulated formulation through ${{\ell }_{2,1}}$-norm regularization is

\begin{equation}\label{L21LDFSraw}
\underset{\mathbf{A}\in {{\mathbb{R}}^{d\times l}}}{\mathop{\min }}\,-tr({{({{\mathbf{A}}^{T}}{{\mathbf{S}}_{t}}\mathbf{A})}^{-1}}({{\mathbf{A}}^{T}}{{\mathbf{S}}_{b}}\mathbf{A}))+\gamma {{\left\| \mathbf{A} \right\|}_{2,1}}.
\end{equation}

To avoid arbitrary scaling and the trivial solution of all zeros, we constrain the transformation vectors of LDA to be uncorrelated w.r.t ${\mathbf{S}_{t}}$, i.e.,  ${{\mathbf{A}}^{T}}{{\mathbf{S}}_{t}}\mathbf{A}=\mathbf{I}$, as in the uncorrelated LDA algorithm\cite{IntrStaPR}. Then the formulation of ${{\ell }_{2,1}}$-norm regularized DFS becomes
\begin{equation}\label{L21LDFS}
\underset{\mathbf{A}\in {{\mathbb{R}}^{d\times l}},{{\mathbf{A}}^{T}}{{\mathbf{S}}_{t}}\mathbf{A}=\mathbf{I}}{\mathop{\min }}\,-tr({{\mathbf{A}}^{T}}{{\mathbf{S}}_{b}}\mathbf{A})+\gamma {{\left\| \mathbf{A} \right\|}_{2,1}}.
\end{equation}

Once we have got the optimal transformation matrix ${{\mathbf{A}}^{*}}$, we can rank each feature ${{\mathbf{f}}_{i}}$ according to ${{\left\| {{\mathbf{a}}^{*i}} \right\|}_{2}}$ in descending order and select the top ranked features, where ${\mathbf{a}}^{*i}$ is the $i$-th row of ${{\mathbf{A}}^{*}}$.

The formulation of DFS in (\ref{L21LDFS}) not only avoids the trivial solution but also preserves the advantage of the original LDFS: it can optimize for feature relevance and redundancy removal automatically \cite{icml10_LDFS}.
More specifically, the features are selected by using the linear transformation matrix ${{\mathbf{A}}^{*}}$, which is learned by the algorithm, so the combinations of these features can lead to the optimal value of the objective of LDA. Hence, DFS selects the most discriminative features, and also the interactions among the features are taken into consideration.
Furthermore, the proposed formulation discards redundant features automatically. Adding redundant features, i.e., features that are correlated to the discriminative features, into the selected feature subset, will not decrease the value of $-tr({{\mathbf{A}}^{T}}{{\mathbf{S}}_{b}}\mathbf{A})$, but will increase ${{\left\| \mathbf{A} \right\|}_{2,1}}$ and then increase the objective value of DFS. Therefore, in the process of minimizing the objective function, the redundant features will be eliminated automatically by DFS.

\subsection{ ${{\ell }_{2,p}}$-Norm Regularized Discriminative Feature Selection}

Recall that selecting fewer features means enforcing the transformation matrix $\mathbf{A}$ to have more zero rows. Thus, the exact formulation of LDA-based feature selection is
\begin{equation}\label{L20Norm}
\underset{\mathbf{A}\in {{\mathbb{R}}^{d\times l}},{{\mathbf{A}}^{T}}{{\mathbf{S}}_{t}}\mathbf{A}=\mathbf{I}}{\mathop{\min }}\,-tr({{\mathbf{A}}^{T}}{{\mathbf{S}}_{b}}\mathbf{A})+\gamma {{\left\| \mathbf{A} \right\|}_{2,0}},
\end{equation}
which is an ${{\ell }_{2,0}}$-norm minimization problem.
However, the problem (\ref{L20Norm}) is difficult to solve as it is a NP-hard combinational optimization problem.

Extensive computational studies have showed that using ${{\ell }_{p}}$-norm ($0<p<1$) can find sparser solution than using ${{\ell }_{1}}$-norm \cite{Signal_nonconvex,Chartrand_CS2008}.
Naturally, one will expect ${{\ell }_{2,p}}$-norm ($0<p<1$) based minimization to be a better sparsity pattern than ${{\ell }_{2,1}}$-norm. The experimental results in \cite{mixedL2p} show that ${{\ell }_{2,p}}$-norm minimization for some $0<p<1$ does find sparser solution than ${{\ell }_{2,1}}$-norm minimization. Thus, the NP-hard feature selection problem can be relaxed to the following problem:
\begin{equation}\label{L2pLDFS}
\underset{\mathbf{A}\in {{\mathbb{R}}^{d\times l}},{{\mathbf{A}}^{T}}{{\mathbf{S}}_{t}}\mathbf{A}=\mathbf{I}}{\mathop{\min }}\,-tr({{\mathbf{A}}^{T}}{{\mathbf{S}}_{b}}\mathbf{A})+\gamma \left\| \mathbf{A} \right\|_{2,p}^{p},p\to 0.
\end{equation}
Obviously, this formulation reduces to (\ref{L21LDFS}) when $p=1$.

It can be easily found that the value of $p$ balances the sparsity and the convexity of the regularization term. When $p=0$, the regularizer is ideal for feature selection in the sense of producing sparse solutions, but it is not convex. On the other hand, ${\ell}_{2,1}$-norm is the closest convex approximation to the ${\ell}_{2,0}$-norm but the sparsity is weakened. In other words, the closer the value of $p$ is to 0, the better approximation the formulation is to the feature selection problem.

Since (\ref{L2pLDFS}) involves ${\ell}_{2,p}$-norm regularization, it is hard to derive its closed solution directly.
In \cite{Nie_RFSL21}, an iterative algorithm has been proposed to solve the joint ${{\ell}_{2,1}}$-norm minimization problem of both the regression loss function and the regularizer. The convergence of the algorithm is also proved in the same literature.
The similar technique is used in \cite{Schatten_p} to minimize the Schatten $p$-Norm with $0<p\le 2$ for matrix completion. Inspired by these two works, we develop a simple united algorithm to solve our proposed DFS for $p$ in $\left( 0,2 \right]$, which will be introduced in next subsection.

\subsection{${{\ell }_{2,p}}$-Norm Minimization Algorithm}
In this subsection, we present a simple united algorithm to solve our proposed formulation for both the convex regularized case ($1\leq p\leq 2$) and the non-convex regularized case ($0<p<1$).

For convenience, we denote $\mathcal{L}(\mathbf{A})=\left\| \mathbf{A} \right\|_{2,p}^{p}$. Note that the derivative of $\mathcal{L}(\mathbf{A})$ w.r.t $\mathbf{A}$ is
\begin{equation}
\frac{\partial \mathcal{L}(\mathbf{A})}{\partial \mathbf{A}}=2\mathbf{DA},
\end{equation}
where $\mathbf{D}\in\mathbb{R}^{d\times d} $ is a diagonal matrix with the  $i$-th diagonal element as
\begin{equation}\label{Defi_D}
{{d}_{ii}}=\frac{p}{2}\left\| {{\mathbf{a}}^{i}} \right\|_{2}^{p-2}.
\end{equation}
When $\mathbf{D}$ is fixed, the derivative in (\ref{L2pLDFS}) can also be regarded as the derivative of the following objective function:
\begin{equation}
-tr({{\mathbf{A}}^{T}}{{\mathbf{S}}_{b}}\mathbf{A})+\gamma tr({{\mathbf{A}}^{T}}\mathbf{DA}).
\end{equation}
Consequently, the problem in (\ref{L2pLDFS}) can be addressed by solving the following problem iteratively:
\begin{equation}\label{trace1}
\underset{\mathbf{A}\in {{\mathbb{R}}^{d\times l}},{{\mathbf{A}}^{T}}{{\mathbf{S}}_{t}}\mathbf{A}=\mathbf{I}}{\mathop{\min }}\,-tr({{\mathbf{A}}^{T}}{{\mathbf{S}}_{b}}\mathbf{A})+\gamma tr({{\mathbf{A}}^{T}}\mathbf{DA}),
\end{equation}
where $\mathbf{D}$ is defined as in (\ref{Defi_D}).

Rewrite (\ref{trace1}), we get
\begin{equation}\label{trace2}
\underset{\mathbf{A}\in {{\mathbb{R}}^{d\times l}},{{\mathbf{A}}^{T}}{{\mathbf{S}}_{t}}\mathbf{A}=\mathbf{I}}{\mathop{\min }}\,tr({{\mathbf{A}}^{T}}(\gamma \mathbf{D}-{{\mathbf{S}}_{b}})\mathbf{A}).
\end{equation}
Solving problem (\ref{trace2}) is equivalent to find the $l$ eigenvectors associated with the minimum $l$ eigenvalues of the following generalized eigen-problem:
\begin{equation}\label{Geig}
(\gamma \mathbf{D}-{{\mathbf{S}}_{b}})\mathbf{a}=\lambda {{\mathbf{S}}_{t}}\mathbf{a}.
\end{equation}

Note that $\mathbf{D}$ is dependent on $\mathbf{A}$ and thus $\mathbf{D}$ is also an unknown variable. We propose an iterative algorithm in this paper to obtain the solution $\mathbf{A}$ such that (\ref{trace2}) is satisfied, and prove in the next section that the proposed iterative algorithm will monotonically decreases the objective of the problem in (\ref{L2pLDFS}).

The algorithm is described in Algorithm 1. In each iteration, $\mathbf{A}$ is calculated with the current $\mathbf{D}$, and then $\mathbf{D}$ is updated based on the current calculated $\mathbf{A}$. The iteration procedure is repeated until the algorithm converges.

\begin{algorithm}[H]
\caption{DFS}
\label{alg1:L2pLDFS}
\begin{algorithmic}
\STATE \textbf{Input:} Data matrix $\mathbf{X}$, label information, parameters: $\gamma$, $l$, $p$.
\STATE \textbf{Output:} $\mathbf{A}\in {{\mathbb{R}}^{d\times l}}$. \\
\STATE 1: Compute the scatter matrix ${\mathbf{S}}_{t}$, ${\mathbf{S}}_{b}$;\\
\STATE 2: Set $k=0$. Initialize ${{\mathbf{D}}_{k}}\in {{\mathbb{R}}^{d\times d}}$ as an identity matrix;\\
\STATE \textbf{Repeat} \\
\STATE 3: Solve the generalized eigen-problem $(\gamma {{\mathbf{D}}_{k}}-{{\mathbf{S}}_{b}})\mathbf{a}=\lambda {{\mathbf{S}}_{t}}\mathbf{a}$; \\
\STATE 4: ${{\mathbf{A}}_{k+1}}=[{{\mathbf{a}}_{1}},{{\mathbf{a}}_{2}},\cdots ,{{\mathbf{a}}_{l}}]$, where ${{\mathbf{a}}_{1}},{{\mathbf{a}}_{2}},\cdots ,{{\mathbf{a}}_{l}}$ are the eigenvectors associated with the first $l$ smallest eigenvalues;\\
\STATE 5: Calculate the diagonal matrix ${{\mathbf{D}}_{k+1}}$, where the $i$-th diagonal element is $\frac{p}{2}\left\| {{\mathbf{a}}_{k+1}^{i}} \right\|_{2}^{p-2}$; \\
\STATE 6: $k = k+1$;\\
\STATE \textbf{Until converges}
\end{algorithmic}
\end{algorithm}

\begin{remark}
To get a stable solution of the generalized eigen-problem (\ref{Geig}), ${{\mathbf{S}}_{t}}$ is required to be nonsingular. This is clearly not true when the number of features is larger than the number of samples. We apply the idea of regularization, by adding some constant values to the diagonal elements of  ${{\mathbf{S}}_{t}}$ as ${{\mathbf{S}}_{t}}+\alpha \mathbf{I}$ for some $\alpha >0$. It is easy to see that ${{\mathbf{S}}_{t}}+\alpha \mathbf{I}$ is nonsingular.
\end{remark}

\begin{remark}
When computing ${{\mathbf{D}}}$, its diagonal element ${{d}_{ii}}$ is $\frac{p}{2}\left\| \mathbf{a}^{i} \right\|_{2}^{p-2}$. In practice, ${{\left\| \mathbf{a}^{i} \right\|}_{2}}$ could be very close to zero but not zero. However, ${{\left\| \mathbf{a}^{i} \right\|}_{2}}$ can be zero theoretically. In this case, $d_{ii}=0$ is a subgradient of $\left\| \mathbf{A} \right\|_{2,p}^{p}$ w.r.t $\mathbf{a}^{i}$. We can not set $d_{ii}=0$ when $\mathbf{a}^{i}=0$, otherwise the derived algorithm will not be guaranteed to converge. Instead, we regularize ${{d}_{ii}}$ as ${{d}_{ii}}=\frac{p}{2}{{\left( {{({{a}^{i}})}^{T}}{{a}^{i}}+\zeta  \right)}^{\frac{p}{2}-1}}$, and the derived algorithm can be proved to minimize $-tr({{\mathbf{A}}^{T}}{{\mathbf{S}}_{b}}\mathbf{A})+\gamma\sum\limits_{i=1}^{d}{{{\left( {{({{\mathbf{a}}^{i}})}^{T}}{{\mathbf{a}}^{i}}+\zeta  \right)}^{\frac{p}{2}}}}$ instead of $-tr({{\mathbf{A}}^{T}}{{\mathbf{S}}_{b}}\mathbf{A})+ \gamma \sum\limits_{i=1}^{d}{{{\left( {{({{\mathbf{a}}^{i}})}^{T}}{{\mathbf{a}}^{i}} \right)}^{\frac{p}{2}}}}=-tr({{\mathbf{A}}^{T}}{{\mathbf{S}}_{b}}\mathbf{A})+\gamma\left\| \mathbf{A} \right\|_{2,p}^{p}$.
It is easy to see that $\sum\limits_{i=1}^{d}{{{\left( {{({{\mathbf{a}}^{i}})}^{T}}{{\mathbf{a}}^{i}}+\zeta  \right)}^{\frac{p}{2}}}}$ approximates $\left\| \mathbf{A} \right\|_{2,p}^{p}$ when $\zeta\to 0$.
\end{remark}

\section{Discussions}
This section gives an analysis of DFS in three aspects. We first provide the convergence behavior of the algorithm and then discuss time complexity and parameter determination problems.

\subsection{Algorithm Analysis}
In this subsection, we prove that the objective function of (\ref{L2pLDFS}) is non-increasing under the updating rules of $\mathbf{A}$ and $\mathbf{D}$ in Algorithm \ref{alg1:L2pLDFS}.
Firstly, the following lemma is introduced.

\begin{lemma}\label{lemma1}
When $0<p\le 2$, for any nonzero vectors $\mathbf{a}$, ${{\mathbf{a}}_{k}}$, the following inequality holds:
\begin{equation}\label{Le1}
\frac{\left\| \mathbf{a} \right\|_{2}^{p}}{\left\| {{\mathbf{a}}_{k}} \right\|_{2}^{p}}-\frac{p}{2}\frac{\left\| \mathbf{a} \right\|_{2}^{2}}{\left\| {{\mathbf{a}}_{k}} \right\|_{2}^{2}}\le 1-\frac{p}{2}.
\end{equation}
\end{lemma}

\begin{proof}
Denote $\varphi (t)={{t}^{p}}-\frac{p}{2}{{t}^{2}}+\frac{p}{2}-1$, then we have
\[{\varphi }'(t)=p{{t}^{p-1}}-pt=pt({{t}^{p-2}}-1).\]
Obviously, when $t>0$ and $0<p\le 2$, $t=1$ is the only point so that ${\varphi }'(t)=0$. Note that ${\varphi }'(t)>0$ ($0<t<1$) and ${\varphi }'(t)<0$ ($t>1$), so $t=1$ is the maximum point. As $\varphi (1)=0$, thus when $t>0$ and $0<p\le 1$, $\varphi (t)\le 0$.
Therefore, let ${{t}^{*}}=\frac{{{\left\| \mathbf{a} \right\|}_{2}}}{{{\left\| {{\mathbf{a}}_{k}} \right\|}_{2}}}$ in $\varphi (t)$, then $\varphi (\frac{{{\left\| \mathbf{a} \right\|}_{2}}}{{{\left\| {{\mathbf{a}}_{k}} \right\|}_{2}}})\le 0$, that is to say
	\[\frac{\left\| \mathbf{a} \right\|_{2}^{p}}{\left\| {{\mathbf{a}}_{k}} \right\|_{2}^{p}}-\frac{p}{2}\frac{\left\| \mathbf{a} \right\|_{2}^{2}}{\left\| {{\mathbf{a}}_{k}} \right\|_{2}^{2}}+\frac{p}{2}-1\le 0.\]
After a transposition, we arrive at (\ref{Le1}).
\end{proof}

\begin{theorem}
When $0<p\le 2$, the Algorithm \ref{alg1:L2pLDFS} will monotonically decrease the objective of the problem in (\ref{L2pLDFS}) in each iteration, and converge to the local optimum of the problem.
\end{theorem}

\begin{proof}
In the $k$-th iteration
\begin{equation}
	{{\mathbf{A}}_{k+1}}=\underset{\mathbf{A}\in {{\mathbb{R}}^{d\times l}},{{\mathbf{A}}^{T}}{{\mathbf{S}}_{t}}\mathbf{A}=\mathbf{I}}{\mathop{\arg \min }}\,-tr(\mathbf{A}^{T}{{\mathbf{S}}_{b}}{{\mathbf{A}}})+\gamma tr(\mathbf{A}^{T}{{\mathbf{D}}_{k}}{{\mathbf{A}}}),
\end{equation}
which indicates that
\begin{equation}
\begin{split}
& -tr(\mathbf{A}_{k+1}^{T}{{\mathbf{S}}_{b}}\mathbf{A}_{k+1}^{T})+\gamma tr(\mathbf{A}_{k+1}^{T}{{\mathbf{D}}_{k}}\mathbf{A}_{k+1}^{T}) \\
& \le -tr(\mathbf{A}_{k}^{T}{{\mathbf{S}}_{b}}\mathbf{A}_{k}^{T})+\gamma tr(\mathbf{A}_{k}^{T}{{\mathbf{D}}_{k}}{{\mathbf{A}}_{k}}).
\end{split}
\end{equation}
That is to say,
\begin{equation}\label{Proofth1}
\begin{split}
& -tr(\mathbf{A}_{k+1}^{T}{{\mathbf{S}}_{b}}\mathbf{A}_{k+1}^{T})+\gamma \sum\limits_{i=1}^{d}{\frac{p}{2}\frac{\left\| \mathbf{a}_{k+1}^{i} \right\|_{2}^{2}}{\left\| \mathbf{a}_{k}^{i} \right\|_{2}^{2-p}}} \\
& \le -tr(\mathbf{A}_{k}^{T}{{\mathbf{S}}_{b}}\mathbf{A}_{k}^{T})+\gamma \sum\limits_{i=1}^{d}{\frac{p}{2}\frac{\left\| \mathbf{a}_{k}^{i} \right\|_{2}^{2}}{\left\| \mathbf{a}_{k}^{i} \right\|_{2}^{2-p}}},
\end{split}
\end{equation}
where vectors $\mathbf{a}_{k}^{i}$ and $\mathbf{a}_{k+1}^{i}$ denote the $i$-th row of matrix ${{\mathbf{A}}_{k}}$ and ${{\mathbf{A}}_{k+1}}$ respectively.
On the other hand, according to Lemma \ref{lemma1}, for each $i$ we have
\begin{equation}
\frac{\left\| \mathbf{a}_{k+1}^{i} \right\|_{2}^{p}}{\left\| \mathbf{a}_{k}^{i} \right\|_{2}^{p}}-\frac{p}{2}\frac{\left\| \mathbf{a}_{k+1}^{i} \right\|_{2}^{2}}{\left\| \mathbf{a}_{k}^{i} \right\|_{2}^{2}}\le 1-\frac{p}{2},
\end{equation}
which is equivalent to the following inequality
\begin{equation}
\left\| \mathbf{a}_{k+1}^{i} \right\|_{2}^{p}-\frac{p}{2}\frac{\left\| \mathbf{a}_{k+1}^{i} \right\|_{2}^{2}}{\left\| \mathbf{a}_{k}^{i} \right\|_{2}^{2-p}}\le \left\| \mathbf{a}_{k}^{i} \right\|_{2}^{p}-\frac{p}{2}\frac{\left\| \mathbf{a}_{k}^{i} \right\|_{2}^{2}}{\left\| \mathbf{a}_{k}^{i} \right\|_{2}^{2-p}},
\end{equation}
so the following inequality holds:
\begin{equation}\label{Proofth2}
\sum\limits_{i=1}^{d}{\left( \left\| \mathbf{a}_{k+1}^{i} \right\|_{2}^{p}-\frac{p}{2}\frac{\left\| \mathbf{a}_{k+1}^{i} \right\|_{2}^{2}}{\left\| \mathbf{a}_{k}^{i} \right\|_{2}^{2-p}} \right)}\le \sum\limits_{i=1}^{d}{\left( \left\| \mathbf{a}_{k}^{i} \right\|_{2}^{p}-\frac{p}{2}\frac{\left\| \mathbf{a}_{k}^{i} \right\|_{2}^{2}}{\left\| \mathbf{a}_{k}^{i} \right\|_{2}^{2-p}} \right)}.
\end{equation}
Combining (\ref{Proofth1}) and (\ref{Proofth2}), we have
\begin{equation}
\begin{split}
& -tr(\mathbf{A}_{k+1}^{T}{{\mathbf{S}}_{b}}\mathbf{A}_{k+1}^{T})+\gamma \sum\limits_{i=1}^{d}{\left\| \mathbf{a}_{k+1}^{i} \right\|_{2}^{p}} \\
& \le -tr(\mathbf{A}_{k}^{T}{{\mathbf{S}}_{b}}\mathbf{A}_{k}^{T})+\gamma \sum\limits_{i=1}^{d}{\left\| \mathbf{a}_{k}^{i} \right\|_{2}^{p}}.
\end{split}
\end{equation}
That is to say,
\begin{equation}
\begin{split}
& -tr(\mathbf{A}_{k+1}^{T}{{\mathbf{S}}_{b}}\mathbf{A}_{k+1}^{T})+\gamma \left\| {{\mathbf{A}}_{k+1}} \right\|_{2,p}^{p} \\
& \le -tr(\mathbf{A}_{k}^{T}{{\mathbf{S}}_{b}}\mathbf{A}_{k}^{T})+\gamma \left\| {{\mathbf{A}}_{k}} \right\|_{2,p}^{p}.
\end{split}
\end{equation}

Thus the Algorithm \ref{alg1:L2pLDFS} will monotonically decrease the objective of the problem in (\ref{L2pLDFS}) in each iteration $k$. Note that the
objective function has lower bounds, so the above iteration will converge.
Therefore, the Algorithm \ref{alg1:L2pLDFS} monotonically decreases the objective value in each iteration till the convergence.
\end{proof}

As we use the transformation matrix $\mathbf{A}$ to select features, we also need to make clear the convergence behavior of it. Following \cite{JELSR}, we measure the divergence between two sequential $\mathbf{A}$s by the following metric:
\begin{equation}\label{DefiError}
Div(k)=\sum\limits_{i=1}^{d}{\left| {{\left\| \mathbf{a}_{k+1}^{i} \right\|}_{2}}-{{\left\| \mathbf{a}_{k}^{i} \right\|}_{2}} \right|}.
\end{equation}
The metric defined above acts as an indicator to show whether the final results would be changed drastically.

\subsection{Time Complexity}
To optimize the objective function of DFS, the most time consuming operation is to solve the generalized eigen-problem $(\gamma {{\mathbf{D}}_{t}}-{{\mathbf{S}}_{b}})\mathbf{a}=\lambda {{\mathbf{S}}_{t}}\mathbf{a}$. The time complexity of this operation is $O({{d}^{3}})$ approximately.
Empirical results show that the convergence is fast and only several iterations are needed to converge. Therefore, the proposed method scales well in practice.

\subsection{Parameter Selection}
Parameter selection is of great importance and it is still an open problem. At present, the most commonly used parameter selection method is grid search based on the cross validation accuracy (CV-Acc), i.e., to search the optimal parameter corresponding to the highest CV-Acc. Sometimes, we also determine parameters based on experience \cite{JELSR}.

If we take $p$ in the ${{\ell}_{2,p}}$-norm regularization term as a parameter, then DFS has three parameters: the reduced dimensionality $l$, regularization parameter $\gamma$ and $p$.
As for the reduced dimensionality $l$, we empirically set it to be $c-1$ as in traditional LDA \cite{IntrStaPR}, where $c$ is the number of classes.
The regularization parameter $\gamma $ controls the trade-off between the discrimination and the sparsity. It plays an important role in DFS for feature selection. We determine it by grid search according to CV-Acc and some numerical results are presented to illustrate its impact on DFS.
Then, the value of $p$, which balances the sparsity and convexity of the formulation, is also hard to decide.
As the transformation matrix need to be row-sparse, we only focus on cases that $0<p\le1$, though the algorithm is proved to be convergent when $0<p\le2$.
To simplify the experiment, $p$ is set as 1 when comparing with other feature selection approaches and the performance of DFS with different $p$ values is studied separately.
Finally, the number of selected features, a common parameter for all feature selection methods, is difficult to determine without prior. Hence, we vary this parameter within a certain range and report the corresponding results.

\section{Experiments}

In this section,  experiments are conducted to evaluate the performance of our proposed algorithm.
Firstly, a toy example is displayed to show the ability of DFS to find the discriminative features. Then we compare DFS with five widely used filter-type feature selection methods, and following this, comparisons between DFS with different $p$ values are made. We also evaluate how DFS performs with varying values of the regularization parameter $\gamma$. Finally, convergence analysis and computational time are reported.

\subsection{Data Set Description and Evaluation Metrics}
In our experiments, six diverse public data sets are collected to illustrate the performance of different feature selection approaches.
These data sets include three image data sets, COIL20\footnote{http://www.cs.columbia.edu/CAVE/software/softlib/coil-20.php}, ORL\footnote{http://www.zjucadcg.cn/dengcai/Data/FaceData.html}, and USPS\footnote{http://www.cc.gatech.edu/\~{}lsong/data/icml\_data.zip}, two biological gene expression microarray data sets, Colon Tumor\footnote{http://www.upo.es/eps/bigs/datasets.html} (COLON) and
Lung Cancer\footnote{https://sites.google.com/site/feipingnie/resoure} (LUNG), and one spoken letter recognition data, ISOLET5\footnote{https://archive.ics.uci.edu/ml/datasets/ISOLET}.
All data sets are standardized to be zero-mean and normalized by standard deviation.
We summarize the statistics of the data sets in Table \ref{tabData} and briefly introduce them as follows,

\begin{itemize}
\item COIL20 contains 1440 images of 20 objects. The images of each object were taken 5 degree apart as the object is rotated on a turntable, and each object has 72 images. The size of each image is 32$\times$32 pixels, with 256 gray levels per pixel. Thus, each image is represented by a 1024-dimensional vector.
\item ORL consists of 400 face images. There are 10 different images of each of 40 distinct subjects. For some subjects, the images were taken at different times with varying lighting, different facial expressions and facial details. The original size of each image is 92$\times$112, with 256 grey levels per pixel. In our experiments, the size is reduced to 32$\times$32.
\item The original USPS handwritten digit database contains 9298 images. The size of each image is 16$\times$16 pixels and each image is characterized by a 256-dimensional vector. The version we use in this paper is a balanced random sample of the original data set produced by L. Song et al. \cite{SongLe_HISC}.
\item COLON contains expression levels of 2000 genes taken from 62 different samples. For each sample it is indicated whether it comes from a tumor biopsy or not. 40 samples are normal and the rest of the samples are from tumor biopsy.
\item LUNG is composed of 203 samples in five classes with 139, 21, 20, 6, 17 samples respectively. Each sample has 12600 genes. The genes with standard deviations smaller than 50 expression units were removed and the remaining data set contains 203 samples with 3312 genes.
\item The ISOLET data set was generated by letting 150 subjects speak the name of each letter of the alphabet twice. Hence, from each speaker we have 52 training examples. The data was divided into 5 equal sets of 30 speakers each. We use the data from the fifth part, ISOLET5. As one example in this part is missing, ISOLET5 has 1559 examples in 26 classes with 617 attributes.

\end{itemize}

\begin{table}
    \centering
    \caption{Data Set Description} \label{tabData}
    \begin{tabular}{|l|c|c|c|l|}\hline
    Data set & Size & \# of Features & \# of Classes & Type \\ \hline
    COIL20   &  1440&       1024    &       20      &   Image, Object\\ \hline
    ORL     &  400 &       1024     &       40      &   Image, Face \\ \hline
    USPS     &  730 &       256     &       10      &   Image, Handwritten\\  \hline
    COLON    &  62  &       2000    &       2       &   Microarray, Biological\\ \hline
    LUNG     &  203 &       3312    &       5       &   Microarray, Biological\\ \hline
    ISOLET5   &  1559&       617     &       26     &   Voice, Alphabet \\ \hline
\end{tabular}
\end{table}

To test the quality of the selected features, two metrics are used: accuracy -- the classification accuracy achieved by the classifier using the selected features; redundancy rate -- the redundancy rate contained in the selected features.
An ideal feature selection algorithm should select features that results in high accuracy, while containing few redundant features \cite{ASUfspackage}.

We use the LIBSVM\footnote{https://github.com/cjlin1/libsvm} software to perform classification, which implements the ``one-against-one'' approach for multiclass cases, see more details in \cite{LIBSVM}. Following \cite{Nie_RFSL21, SongLe_HISC}, the SVM classifier is individually performed on each data set with the selected features, using the linear kernel with the parameter $C=1$ and 5-fold cross validation. The average classification accuracy of all these 5 folds is reported as the final result.

Assume $\mathcal{F}$ is the set of selected features, and ${{\mathbf{X}}_{\mathcal{F}}}$ is the data represented by the features in $\mathcal{F}$. We use the following measurement to measure the redundancy rate of $\mathcal{F}$ \cite{ASUfspackage}:
\begin{equation}
    RED(\mathcal{F})=\frac{1}{\left| \mathcal{F} \right|(\left| \mathcal{F} \right|-1)}\sum\limits_{{{\mathbf{f}}_{i}},{{\mathbf{f}}_{j}}\in \mathcal{F},i>j}{cor{{r}_{i,j}}},
\end{equation}
where $\left| \mathcal{F} \right|$ is the cardinality of $\mathcal{F}$, i.e., the number of selected features, and ${cor{{r}_{i,j}}}$ is the correlation between two features, ${{\mathbf{f}}_{i}}$ and ${{\mathbf{f}}_{j}}$.
This measurement assesses the averaged correlation of all feature pairs in $\mathcal{F}$. A large value indicates that many selected features are correlated, and thus redundancy is expected to exist in $\mathcal{F}$.

\subsection{A Toy Example}
In this subsection, we present a toy example to illustrate the ability of DFS to select discriminative features. Specifically, two samples from each class of the ORL data set are randomly selected as the training data, and the rest of the examples are used for testing.
We perform our method on the training data. The top ranked \{32, 64, 128, 256, 384, 512, 640, 768, 896, 1024\} features are selected respectively.
Then, the images of two randomly selected testing samples are recovered by using different numbers of selected features, from 32 to all. For illustration, the unselected features are set to be white and the selected features maintain their original values. The recovered images are displayed in Fig. \ref{toyfig}, from left to right, the number of selected features is \{32, 64, 128, 256, 384, 512, 640, 768, 896, 1024\} respectively.

\begin{figure}
  \centering
  \includegraphics[width=0.48\textwidth]{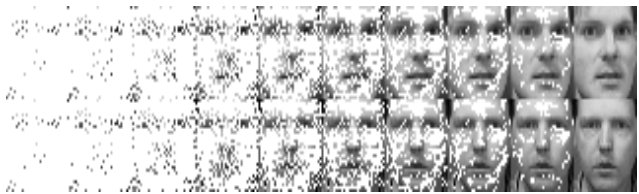}\\
  \caption{A toy example. Top: The test sample of ORL data from the 18th class with different numbers of selected features; Bottom: The test sample of ORL data from the 38th class with different numbers of selected features.}\label{toyfig}
\end{figure}

From Fig. \ref{toyfig}, we draw the following conclusions. The recovered images show that DFS preferentially selects features corresponding to the eyes, nose and mouth. They are the most discriminative features that could describe each individual's character.
We can see that with only 64 features selected, the eyes, nose and mouth of each testing sample are already clear.
While the pixels of the skin are the background, and they have been dropped by our method in most cases.
We also noted that the features selected by DFS are well distributed and do not gather at a certain part of the face. This indicates that DFS manages to remove the redundant features.

\subsection{Comparison between DFS and Other Filter-type Feature Selection Algorithms}

If we tune the values of $\gamma$ and $p$ simultaneously, the experiment will become complicated . For simplicity, we set $p=1$ for DFS when comparing it with other algorithms, since the efficiency of ${\ell}_{2,1}$-norm in feature selection has been demonstrated in many studies \cite{Nie_RFSL21,JELSR,DisUFS}. The effect of $p$ on the performance of DFS will be studied separately in next subsection.
As LDFS has a trivial solution and the proposed implementation in \cite{icml10_LDFS} is hard to reproduce, we do not make comparison with LDFS.
We compare the DFS algorithm with the following widely used filter-type feature selection algorithms,
\begin{itemize}
\item BAHSIC \cite{SongLe_HISC}, which is a backward elimination feature selection method that employs the Hilbert-Schmidt Independence Criterion (HSIC) as a measure of dependence between the features and the labels.
\item Laplacian Score (LS) \cite{LapalcianScore} which evaluates the features according to their ability of preserving the local manifold structure.
\item mRMR \cite{Peng_mRMR}, which selects features that are mutually far away from each other and have ``high" correlation to the classification variable according to the minimal-redundancy-maximal-relevance criterion based on mutual information.
\item ReliefF (RF) \cite{ReliefF}, which evaluates features based on how well the feature differentiates between neighboring instances from different classes versus from the same class.
\item Trace Ratio (TR) \cite{TraceRatio}, which selects a feature subset based on the corresponding subset-level score that is calculated in a trace ratio form.
\end{itemize}

For BAHSIC, a linear kernel is used on both data and labels. We need to tune the bandwidth parameter for the Gaussian kernel in Laplacian Score, and decide the number of the nearest neighbors used per class in ReliefF.
For DFS ($p=1$), the reduced dimensionality $l$ is set as $c-1$, as in traditional LDA\cite{IntrStaPR}. We also need to tune the regularization parameter $\gamma$. We decide the undetermined parameters in a
heuristic way by grid search. For Trace Ratio, the weight matrices are constructed in the same manner as Fisher Score (refer to \cite{TraceRatio} for details).
The implementations of BAHSIC, Laplacian Score and Trace Ratio are downloaded from the authors' websites. The code of mRMR and ReliefF are from the ASU Feature Selection Repository\footnote{http://featureselection.asu.edu/index.php}.
We set the number of selected features between 10 and 100 with an interval of 5 for all data sets. Each feature selection algorithm is first performed to select features.
Then LIBSVM software is employed to classify samples represented by the selected features,  using linear kernel and 5-fold cross validation. We report the average classification accuracy of these as the final result.

\begin{figure*}
\centering
\subfigure[COIL20]{
\includegraphics[width=0.3\textwidth]{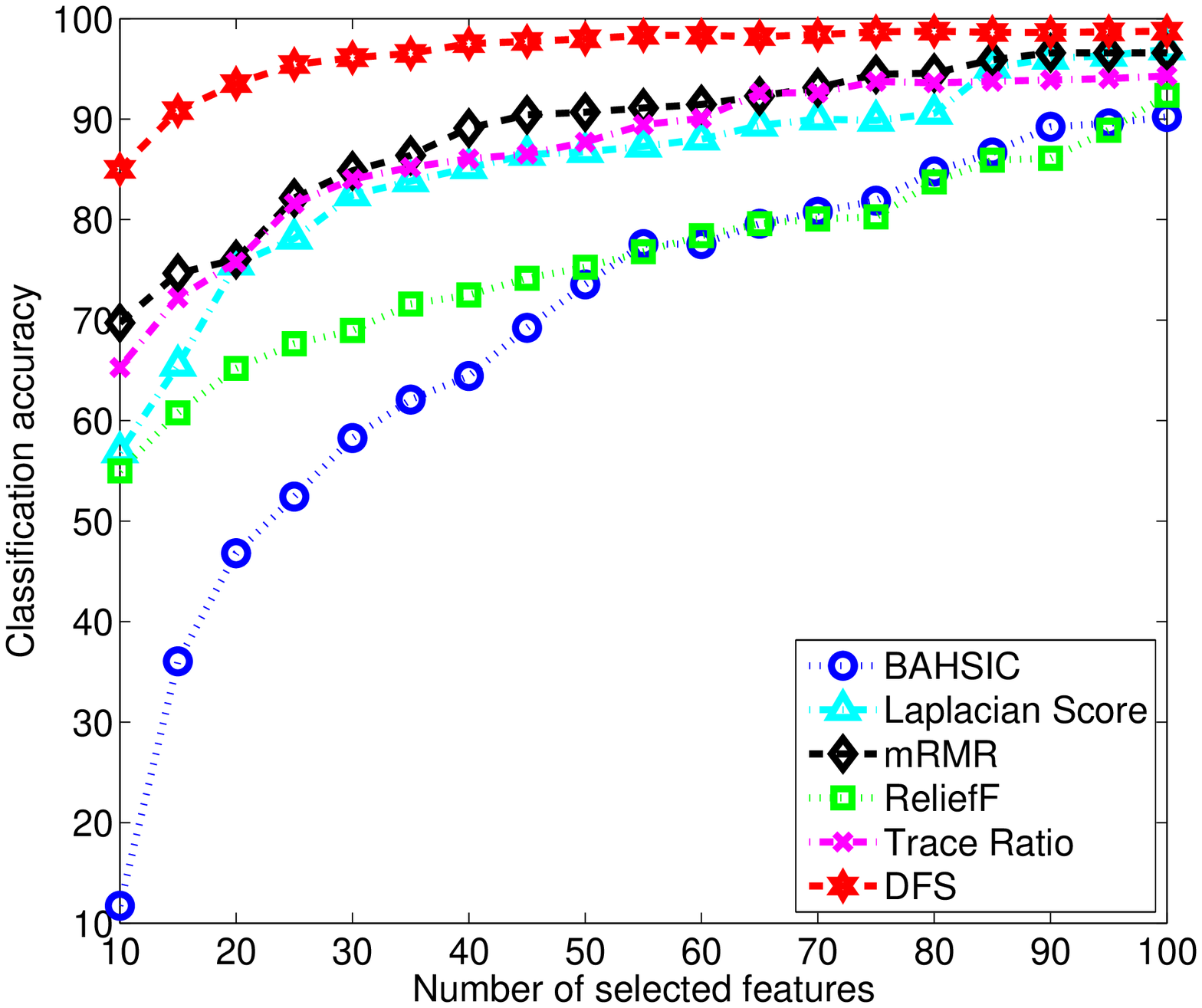}}
\subfigure[ORL]{
\includegraphics[width=0.3\textwidth]{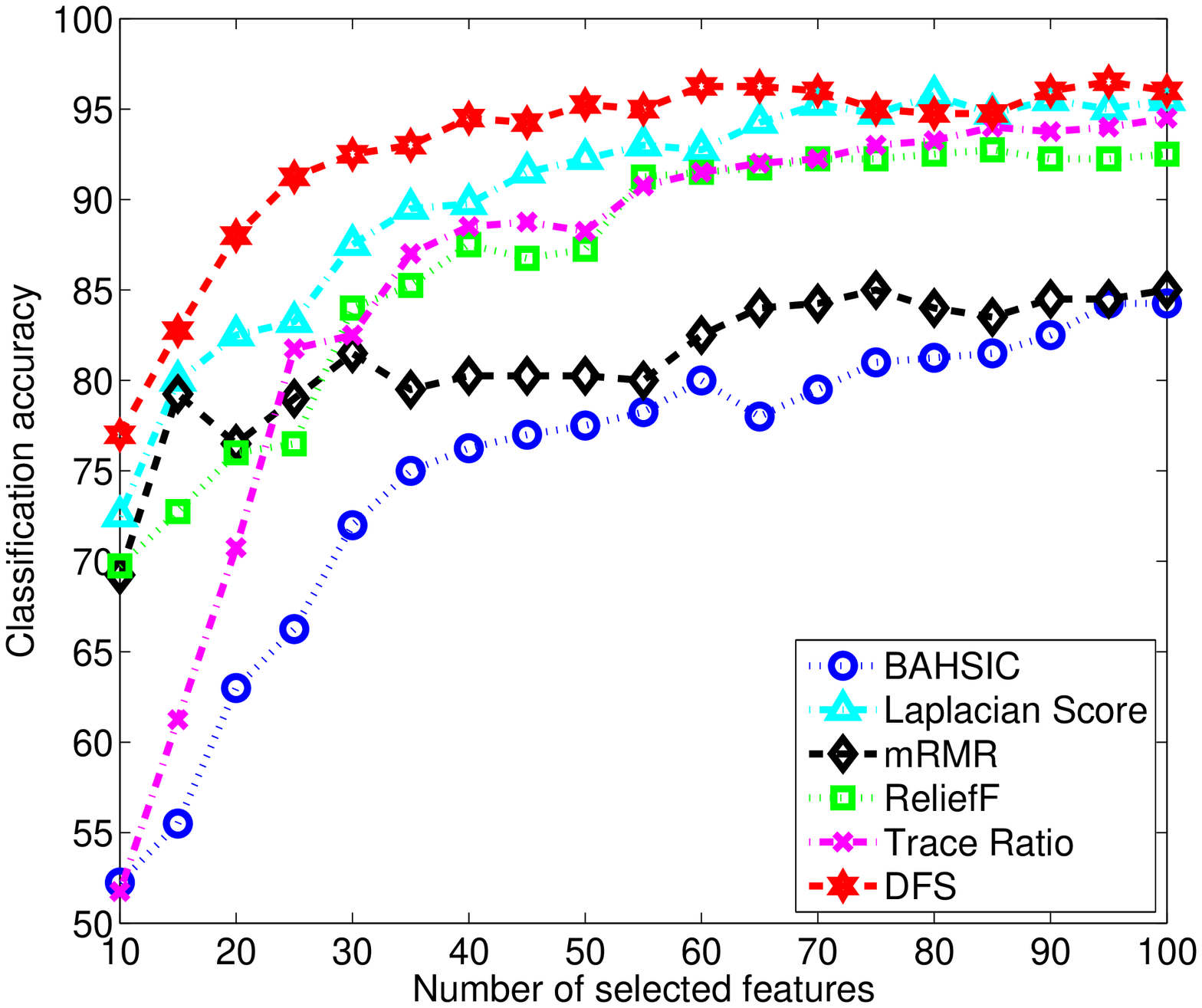}}
\subfigure[USPS]{
\includegraphics[width=0.3\textwidth]{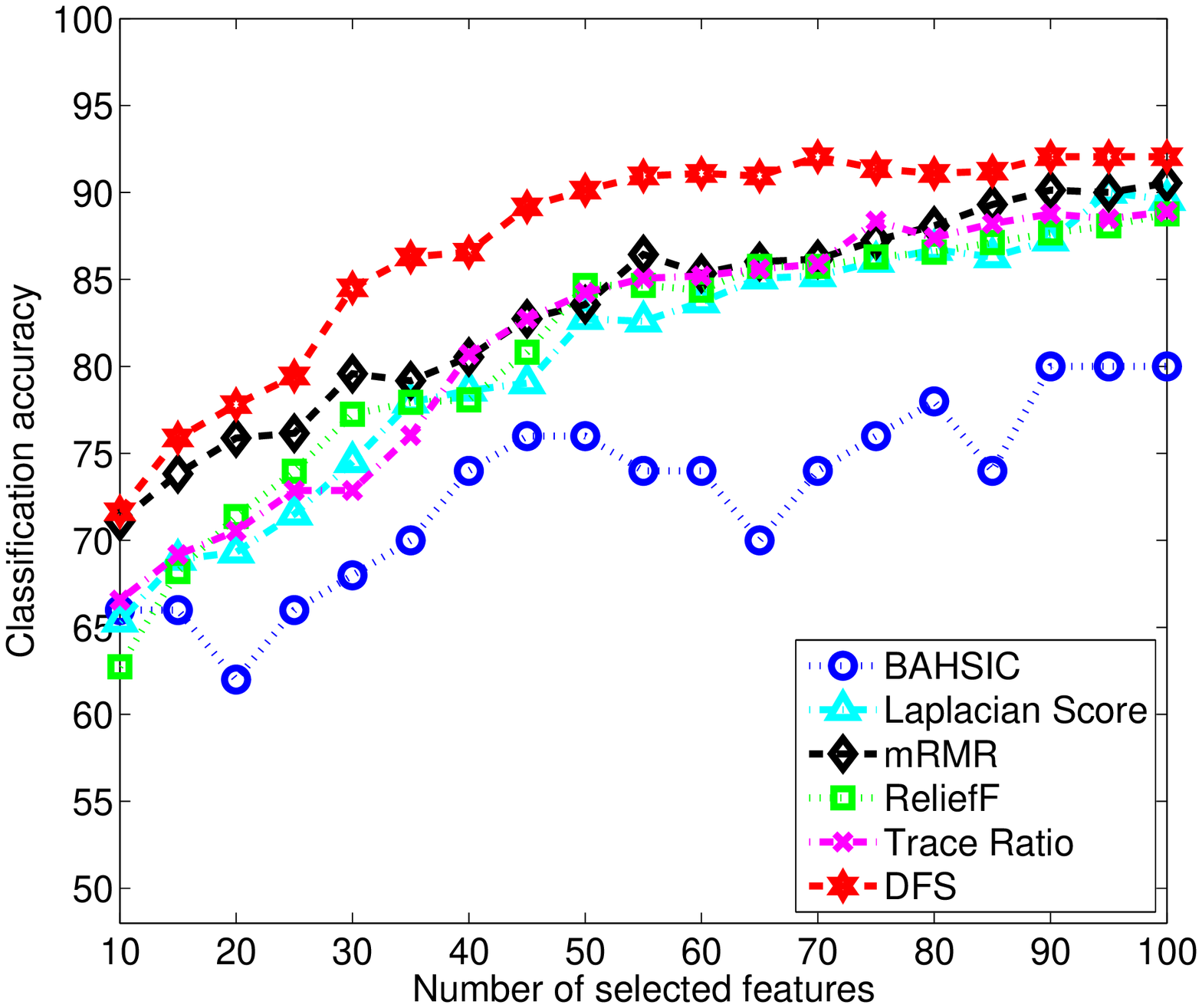}}
\subfigure[COLON]{
\includegraphics[width=0.3\textwidth]{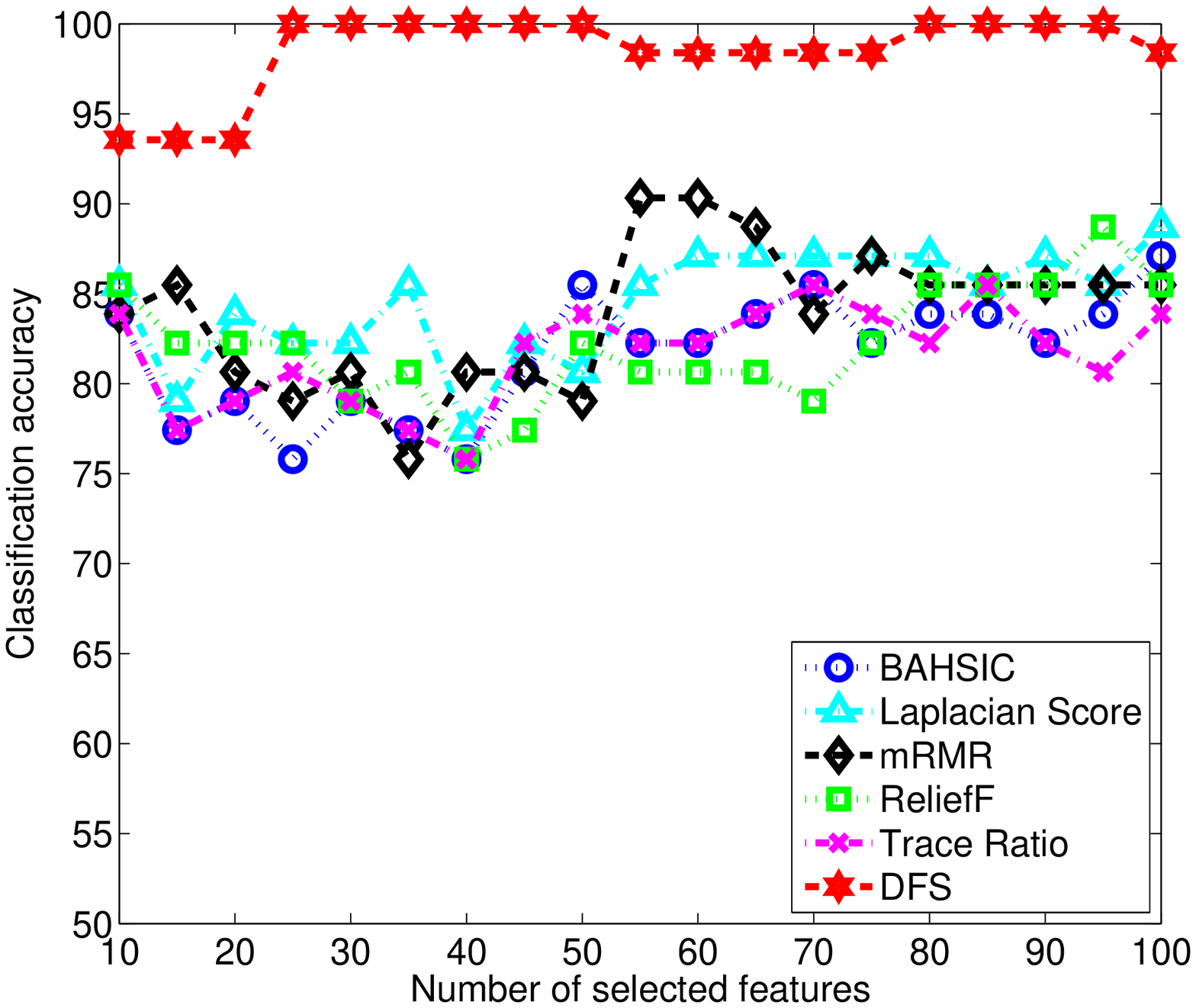}}
\subfigure[LUNG]{
\includegraphics[width=0.3\textwidth]{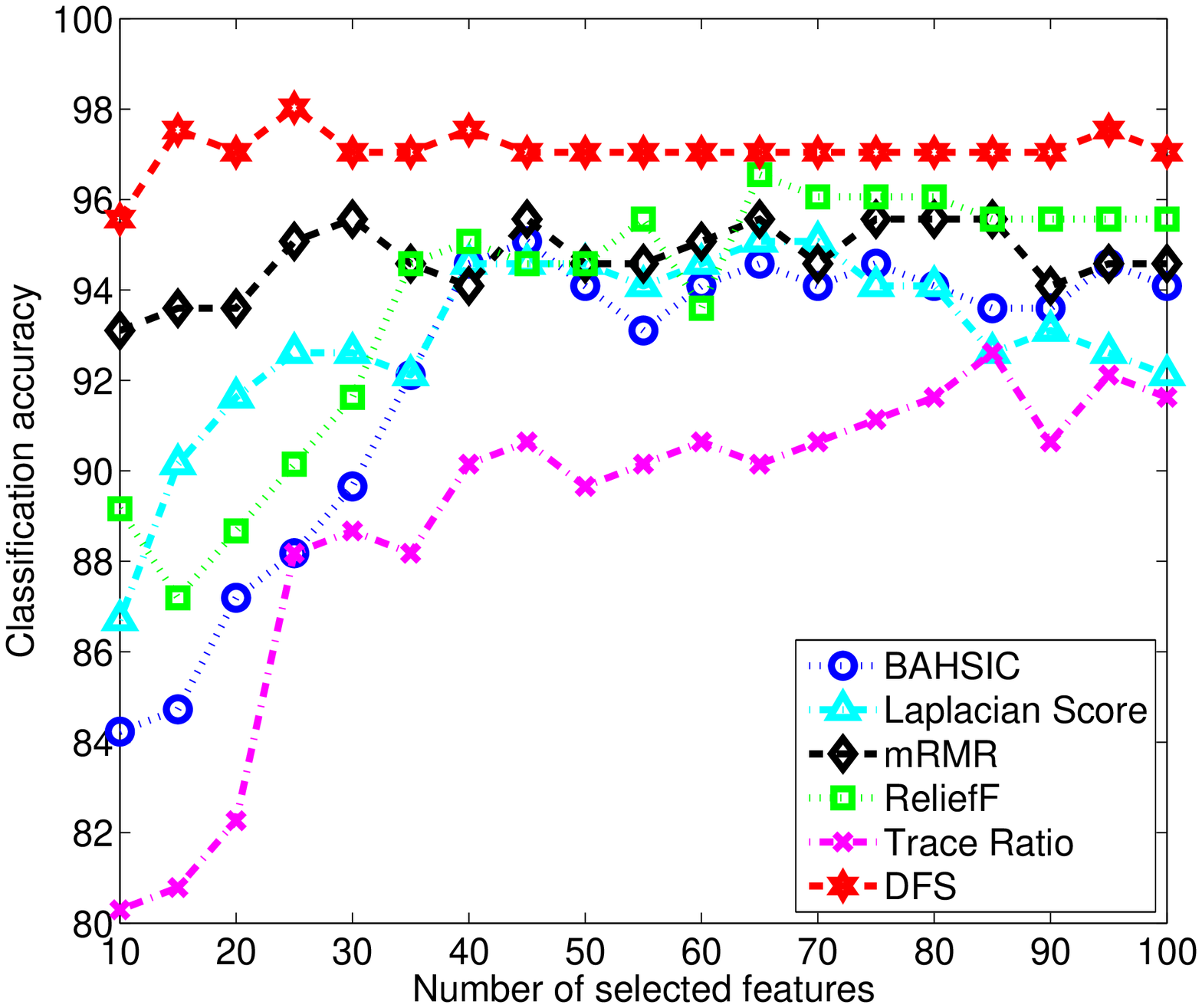}}
\subfigure[ISOLET5]{
\includegraphics[width=0.3\textwidth]{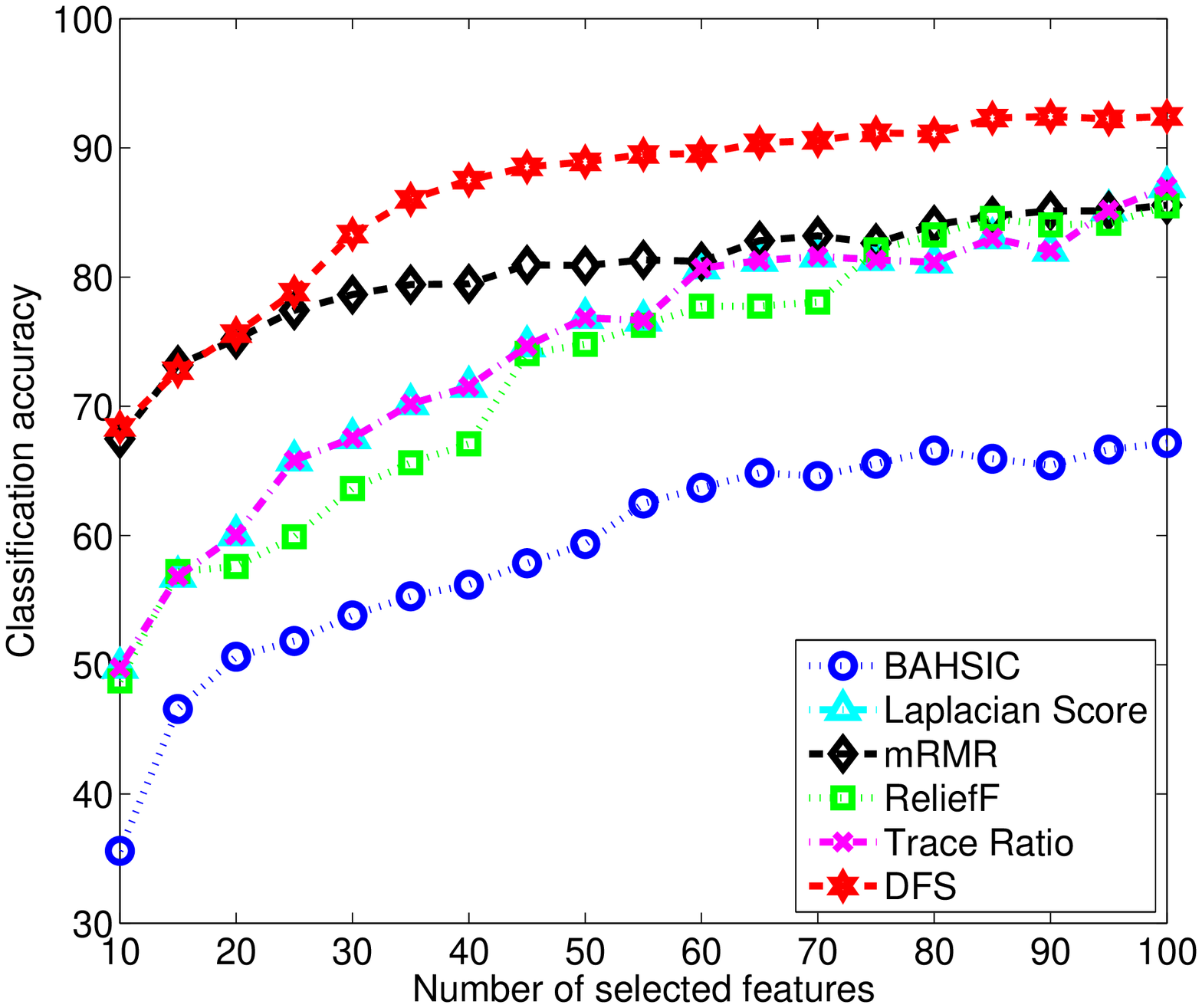}}
\centering
\caption{Comparison between DFS and other filter-type feature selection methods w.r.t classification accuracy. }
\label{Fig1.SVMAcc}
\end{figure*}

\begin{table*}
  \centering
  \caption{Classification Accuracy (\%) of SVM Using 5-Fold Cross Validation for Top 20 and 40 Features Respectively}
  \label{tabAcc1}
  \begin{tabular}{l c c c c c c c c c c c c }\hline
  & \multicolumn{6}{c}{Average accuracy of top 20 features (\%)} & \multicolumn{6}{c}{Average accuracy of top 40 features (\%)}  \\ \cline{2-7} \cline{8-13}
  & BAHSIC & LS & mRMR & RF & TR & DFS & BAHSIC & LS & mRMR & RF & TR & DFS \\ \hline

  COIL20& 46.81 & 75.56 & 76.04 & 65.21 & 75.76 &\textbf{93.54}  & 64.44 & 85.14 & 89.10 & 72.50 & 86.04 & \textbf{97.50} \\
  ORL   & 63    & 82.50 & 76.50 & 76    & 70.75 & \textbf{88}    & 76.25 & 89.75 & 80.25 & 87.50 & 88.50 & \textbf{94.50} \\
  USPS  & 62    & 69.32 & 75.89 & 71.37 & 70.55 & \textbf{77.81} & 74    & 78.63 & 80.55 & 78.08 & 80.68 & \textbf{86.58} \\
  COLON & 79.03 & 83.87 & 80.65 & 82.26 & 79.03 & \textbf{93.55} & 75.81 & 77.42 & 80.65 & 75.81 & 75.81 & \textbf{100}   \\
  LUNG  & 87.19 & 91.63 & 93.60 & 88.67 & 82.27 & \textbf{97.04} & 94.58 & 94.58 & 94.09 & 95.07 & 90.15 & \textbf{97.54} \\
  ISOLET5& 50.61 & 60.04 & 75.18 & 57.60 & 60.04 & \textbf{75.63} & 56.19 & 71.52 & 79.47 & 67.09 & 71.52 & \textbf{87.49}  \\
  Average&64.77 & 77.15 & 79.64 & 73.52 & 73.07 & \textbf{87.60} & 73.55 & 82.84 & 84.02 & 79.34 & 82.12 & \textbf{93.94} \\
  \hline
  \end{tabular}
\end{table*}

\begin{table*}
  \centering
  \caption{Classification Accuracy (\%) of SVM Using 5-Fold Cross Validation for Top 60 and 80 Features Respectively}
  \label{tabAcc2}
  \begin{tabular}{l c c c c c c c c c c c c }\hline
  & \multicolumn{6}{c}{Average accuracy of top 60 features (\%)} & \multicolumn{6}{c}{Average accuracy of top 80 features (\%)}  \\ \cline{2-7} \cline{8-13}
  & BAHSIC & LS & mRMR & RF & TR & DFS & BAHSIC & LS & mRMR & RF & TR & DFS \\  \hline

  COIL20& 77.57 & 87.99 & 91.46 & 78.40 & 90.07 & \textbf{98.33} & 84.79 & 90.56 & 94.58 & 83.75 & 93.61 & \textbf{98.75}  \\
  ORL   & 80    & 92.75 & 82.50 & 91.50 & 91.50 & \textbf{96.25} & 81.25 & \textbf{95.75} & 84    & 92.50 & 93.25 & 94.75 \\
  USPS  & 74    & 83.70 & 85.34 & 84.38 & 85.21 & \textbf{91.10} & 78    & 86.71 & 88.08 & 86.58 & 87.40 & \textbf{91.10}  \\
  COLON & 82.26 & 87.10 & 90.32 & 80.65 & 82.26 & \textbf{98.39} & 83.87 & 87.10 & 85.48 & 85.48 & 82.26 & \textbf{100}    \\
  LUNG  & 94.09 & 94.58 & 95.07 & 93.60 & 90.64 & \textbf{97.04} & 94.09 & 94.09 & 95.57 & 96.06 & 91.63 & \textbf{97.04}  \\
  ISOLET5& 63.69 & 80.69 & 81.21 & 77.74 & 80.69 & \textbf{89.54} & 66.58 & 81.14 & 84.03 & 83.26 & 81.14 & \textbf{91.08}  \\
  Average&78.60 & 87.80 & 87.65 & 84.38 & 86.73  & \textbf{95.11} & 81.43 & 89.23 & 88.62 & 87.94 & 88.22 & \textbf{95.45}  \\
  \hline
  \end{tabular}
\end{table*}

\begin{figure*}
\centering
\subfigure[COIL20]{
\includegraphics[width=0.3\textwidth]{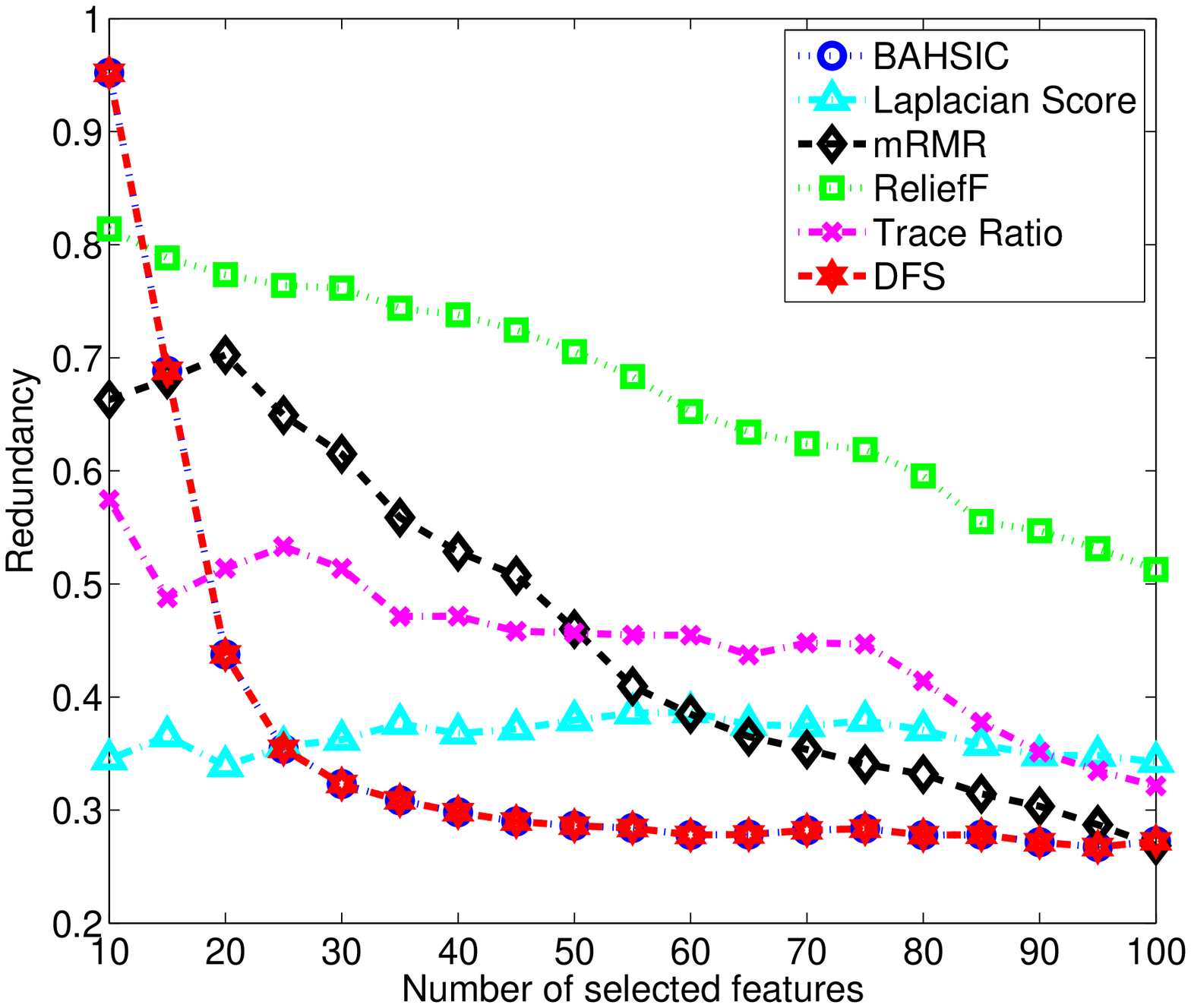}}
\subfigure[ORL]{
\includegraphics[width=0.3\textwidth]{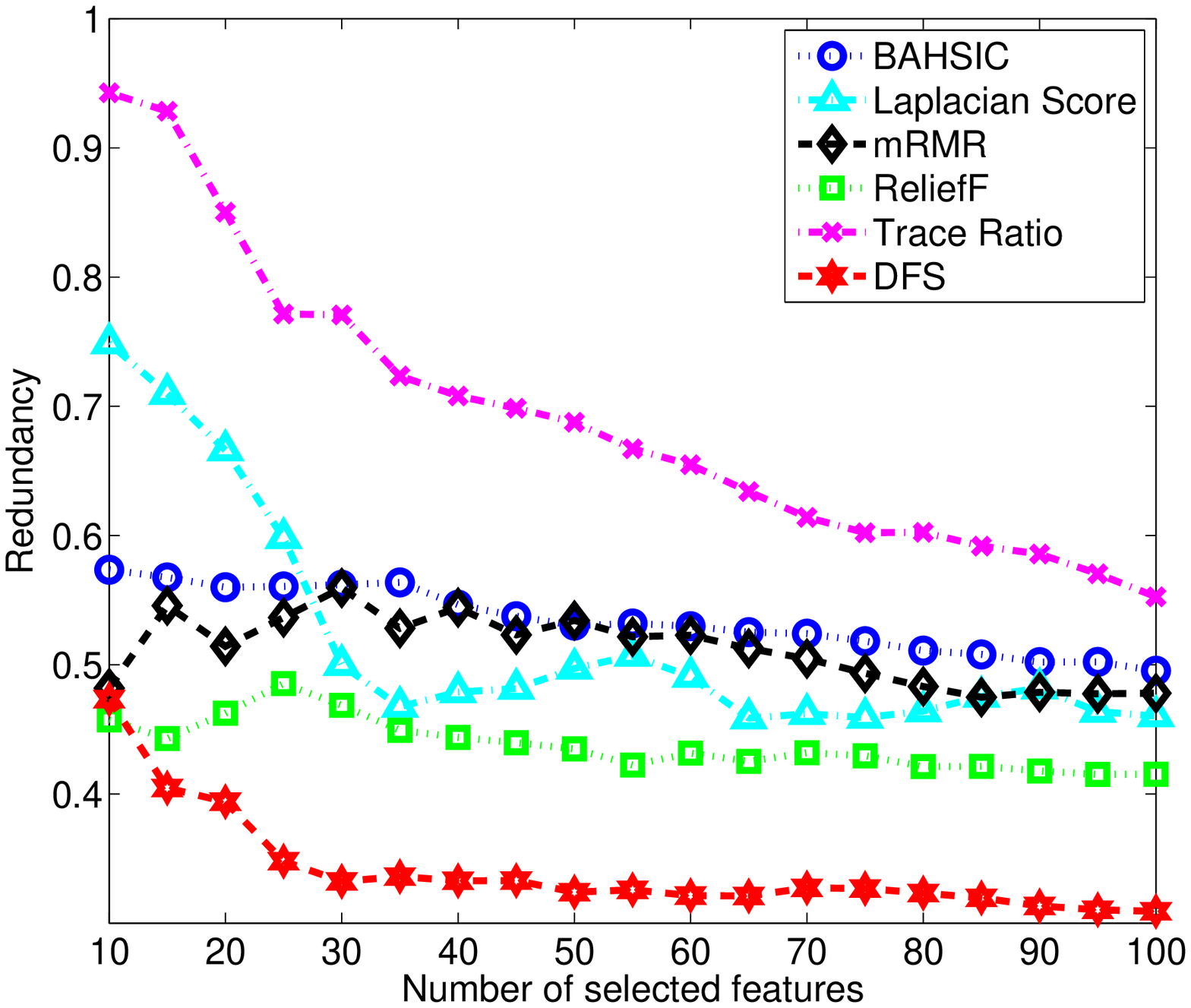}}
\subfigure[USPS]{
\includegraphics[width=0.3\textwidth]{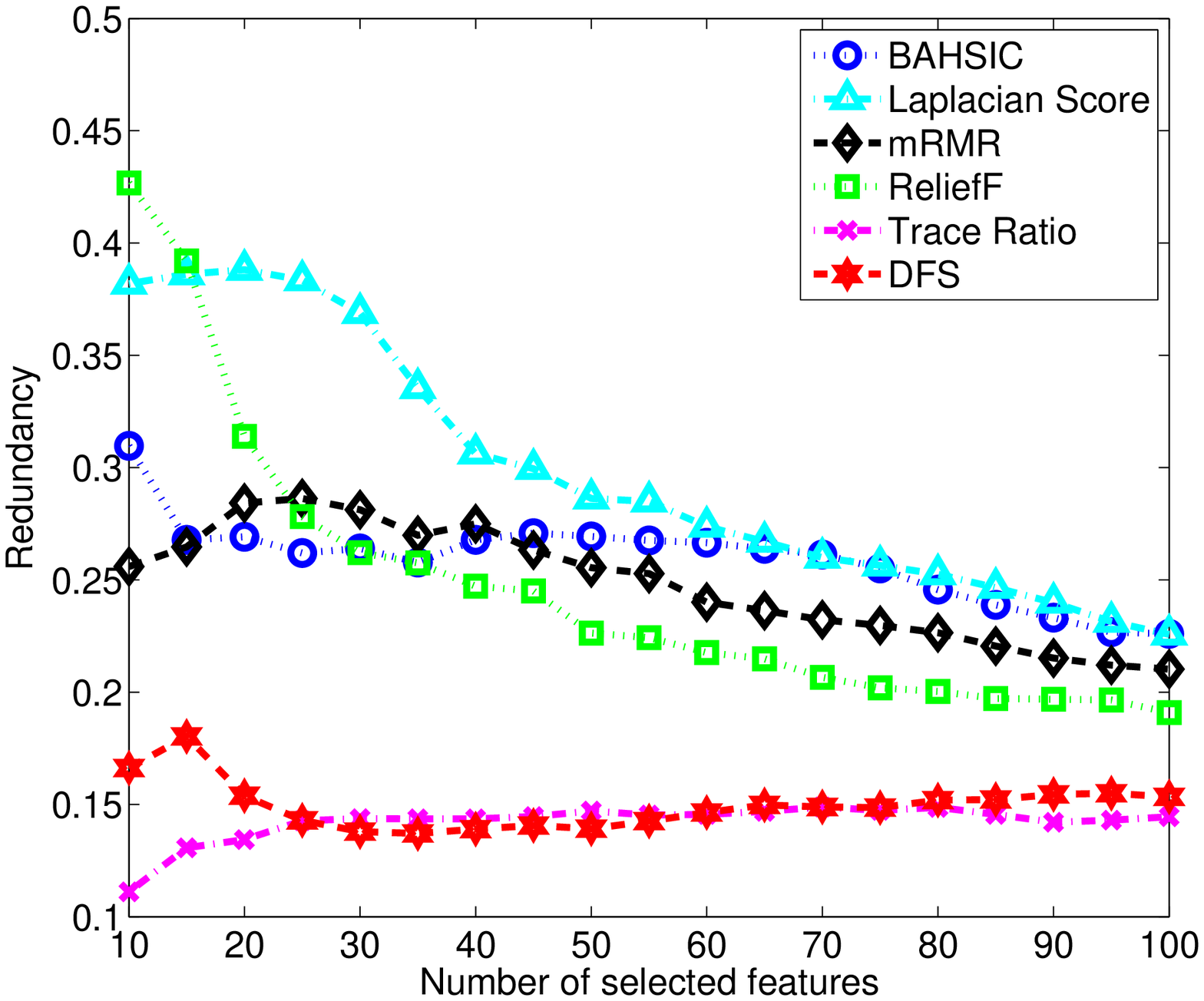}}
\subfigure[COLON]{
\includegraphics[width=0.3\textwidth]{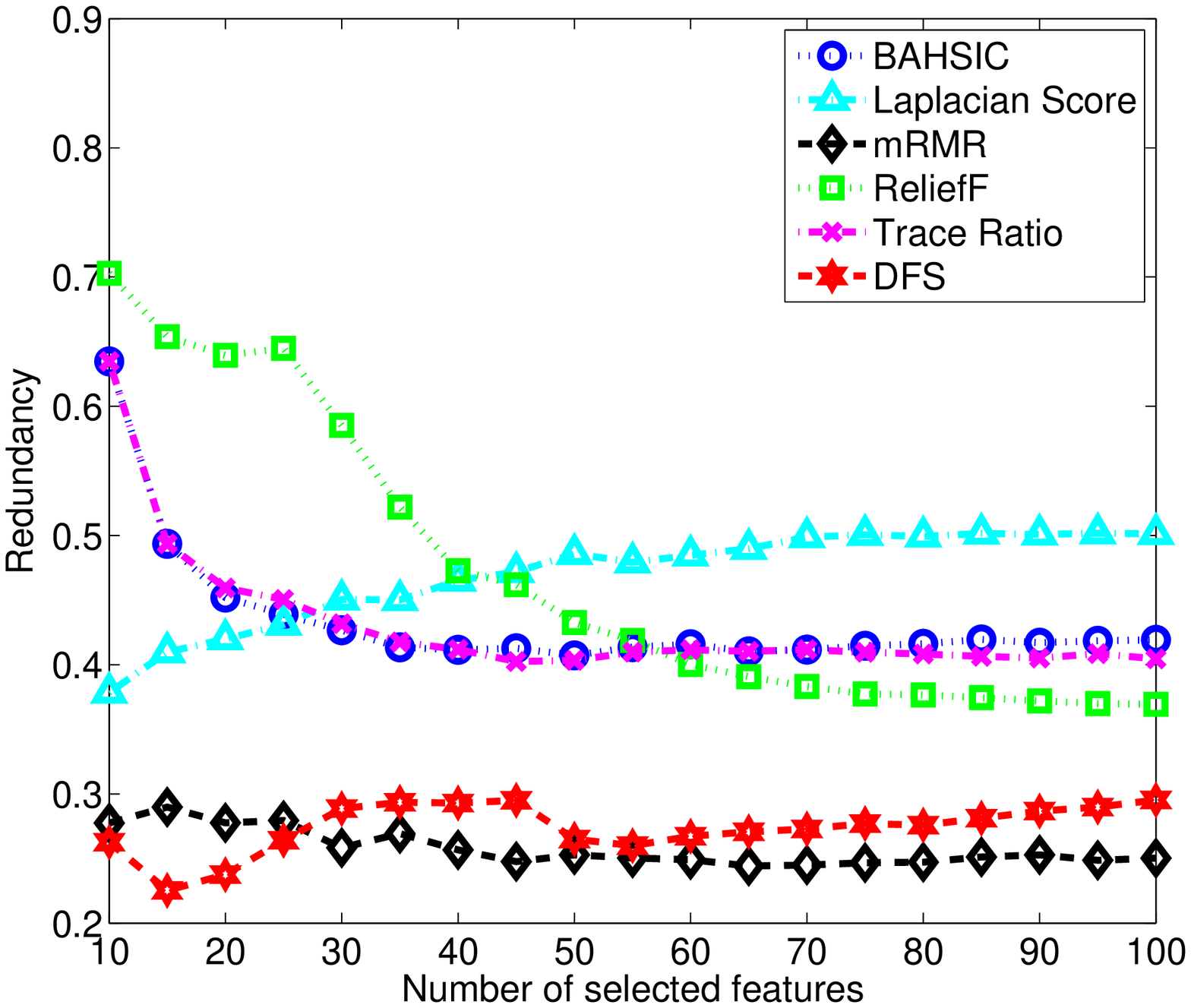}}
\subfigure[LUNG]{
\includegraphics[width=0.3\textwidth]{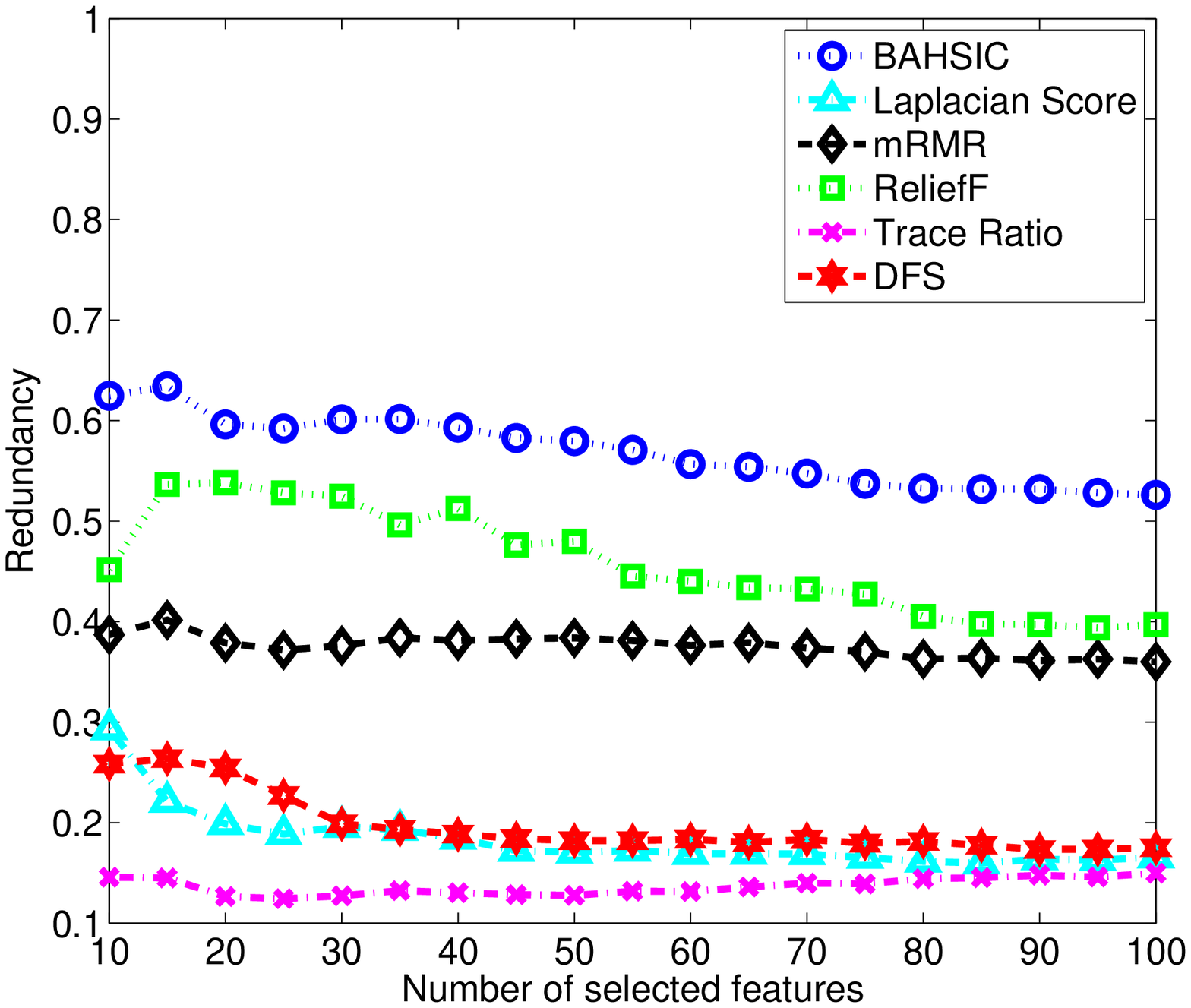}}
\subfigure[ISOLET5]{
\includegraphics[width=0.3\textwidth]{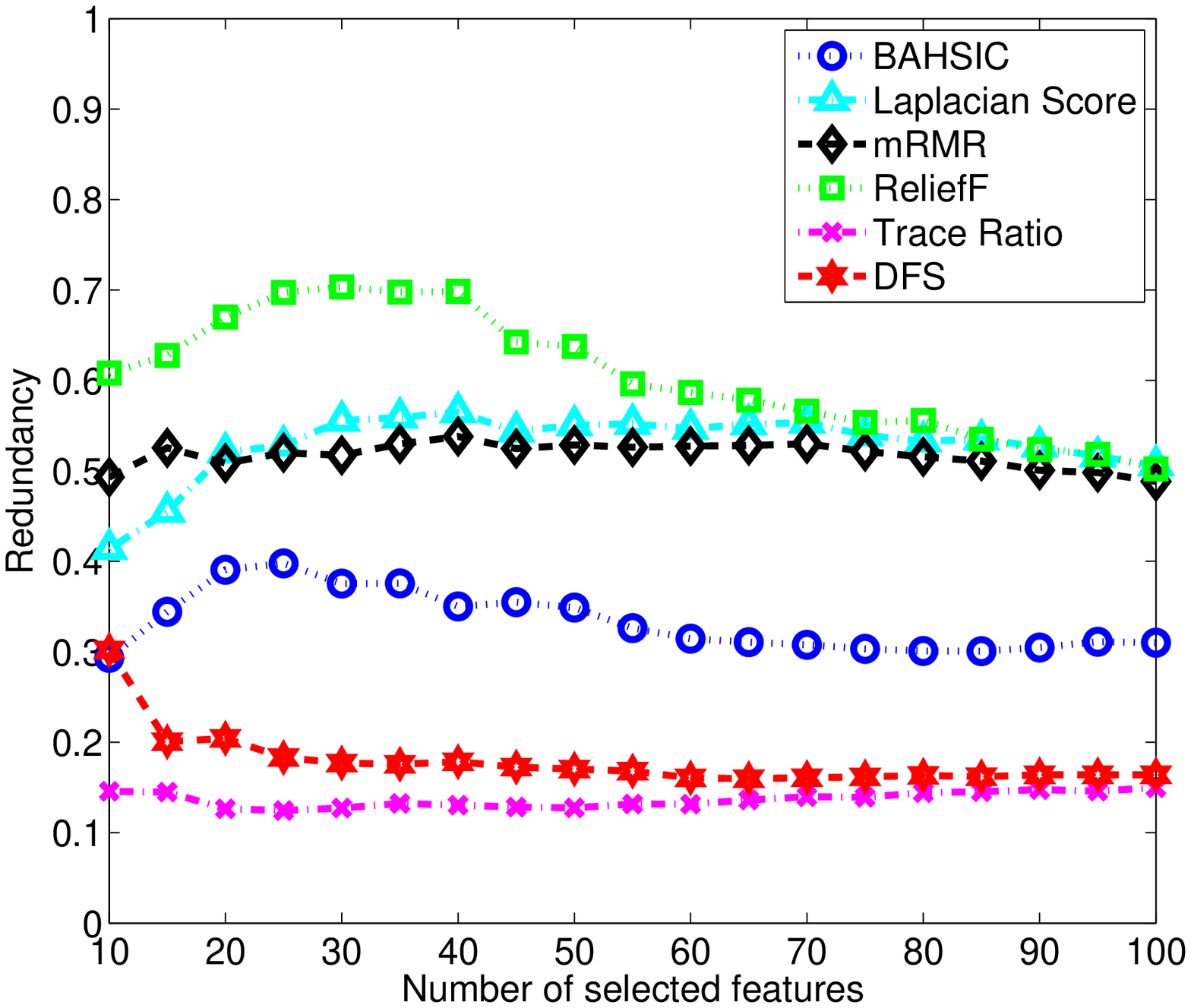}}
\centering
\caption{The redundancy rate of the sets of features of different size selected by different feature selection methods. }
\label{Fig2.Redundancy}
\end{figure*}

Fig. \ref{Fig1.SVMAcc} shows the classification accuracy computed by the SVM classifier on six data sets using different feature selection algorithms.
Table \ref{tabAcc1} and Table \ref{tabAcc2} show the detailed experimental results using the top 20, 40, 60, 80 features respectively. The last line of both tables is the average accuracy over all the data sets for each feature selection approach.
Fig. \ref{Fig2.Redundancy} represents the corresponding redundancy rate when different numbers of features are selected by different feature selection algorithms.

As shown in Fig. \ref{Fig1.SVMAcc}, with the increase in the number of selected features, the trend of classification accuracy of different feature selection methods on different data sets varies. For data sets COIL20, ORL and ISOLET5, all feature selection approaches achieve higher classification accuracy with more selected features. A similar tendency can be found on the USPS data set, with only BAHSIC's performance fluctuating widely. On the data set COLON, the classification accuracy achieved by each method fluctuates within a certain range. On the LUNG data set, DFS and mRMR level off at about 97.00\% and 94.50\% respectively, while the classification accuracy of other methods increase with fluctuation.

Interestingly, on data sets COIL20, USPS and ISOLET5, the classification accuracy achieved by Laplacian Score and Trace Ratio are approximately the same. The reason may be that with a special graph structure, Laplacian Score is equivalent to Fisher Score, and the weight matrices in Trace Ratio are also constructed in the Fisher LDA manner. While on the data sets ORL, COLON and LUNG, Laplacian Score surpasses Trace Ratio.

In terms of the classification accuracy, most of the time DFS outperforms all the baseline methods on all data sets. Especially, on the COLON data set, DFS achieves 8.07\% to 19.35\% improvement compared to the best result of all the other methods. We compute the average classification accuracy over all data sets for each method using the top 20, 40, 60, and 80 features respectively. On average, DFS consistently outperforms the other five methods on all data sets. The mRMR algorithm performs the second best when 20 and 40 features are selected. Laplacian Score replaces mRMR when 60 and 80 features are used. Compared with mRMR or Laplacian Sorce, DFS obtains 7.96\%, 9.92\%, 7.31\% and 6.22\% relative improvement respectively.

From Fig. \ref{Fig2.Redundancy}, we can see that feature subsets selected by DFS on all data sets consistently have a low redundancy rate. DFS selects features whose combination can lead to directions where data points from different classes are far from each other. Similarly, ReliefF's evaluation criterion is to select features that contribute to the separation of the samples from different classes. However, ReliefF shows no advantage in handling feature redundancy.

One point should be highlighted here. As seen from the results, in most cases, the redundancy rate of the selected feature subset decreases as the number of selected features increases. This seems to be counter-intuitive. Redundancy stems from the inter-correlation between the selected features, and thus the total amount of redundancy should increase when the selected feature subset is enlarged.
However, the redundancy rate is calculated by averaging the feature-feature inter correlation coefficients. It indicates the mean redundancy level of the selected feature subset, not the total amount of redundancy. Therefore, the redundancy rate of the enlarged feature subset can be higher or lower than that of the original one. In the case that the classification accuracy can be guaranteed, the main goal of feature selection is to reduce the redundancy as well as the number of features. In practice, the number of selected features is predetermined, so feature selection algorithms are generally designed to seek for high classification accuracy and low redundancy. Compared with the other methods, DFS manages to achieve this goal more successfully.

Combining Fig. \ref{Fig1.SVMAcc} and Fig. \ref{Fig2.Redundancy}, we know that there is no definite relationship between a feature subset's discriminative power and its redundancy rate. That is to say, a feature subset with higher discriminative power does not necessarily have lower redundancy and vice versa.

In summary, DFS, which combines discriminant analysis and ${{\ell }_{2,p}}$-norm regularization, can enhance the feature selection performance in terms of classification. There are two main reasons for this. First, DFS selects features jointly by using the learning mechanism,
hence, the interactions among the whole set of features are considered.
Second, the optimization of DFS impels it to select the most discriminative features and remove the redundant ones simultaneously.

\afterpage{\clearpage}

\subsection{Comparison of DFS with Different $p$ Values}

The value of $p$ balances the sparsity and convexity of the formulation of DFS. The closer to 0 the value of $p$ is, the sparser the representation is. In this subsection, we compare the performance of DFS with different $p$ values.
Note that our goal in extending ${\ell}_{2,1}$-norm regularization to ${\ell}_{2,p}$-norm regularization is to find sparser solutions. We only consider the cases when $0<p\le1$ despite the fact that our algorithm is convergent for all $p\in(0,2]$.
In the experiments, data sets ORL, USPS and ISOLET5 are employed. The value of $p$ is set as \{0.001, 0.01, 0.1, 1\}.
Since DFS with different $p$ values may have different optimal regularization parameters, we tune this parameter for each $p$ and report the best results. Fig. \ref{Fig.pEffect} shows the results.

\begin{figure*}
\centering
\subfigure[ORL]{
\includegraphics[width=0.3\textwidth]{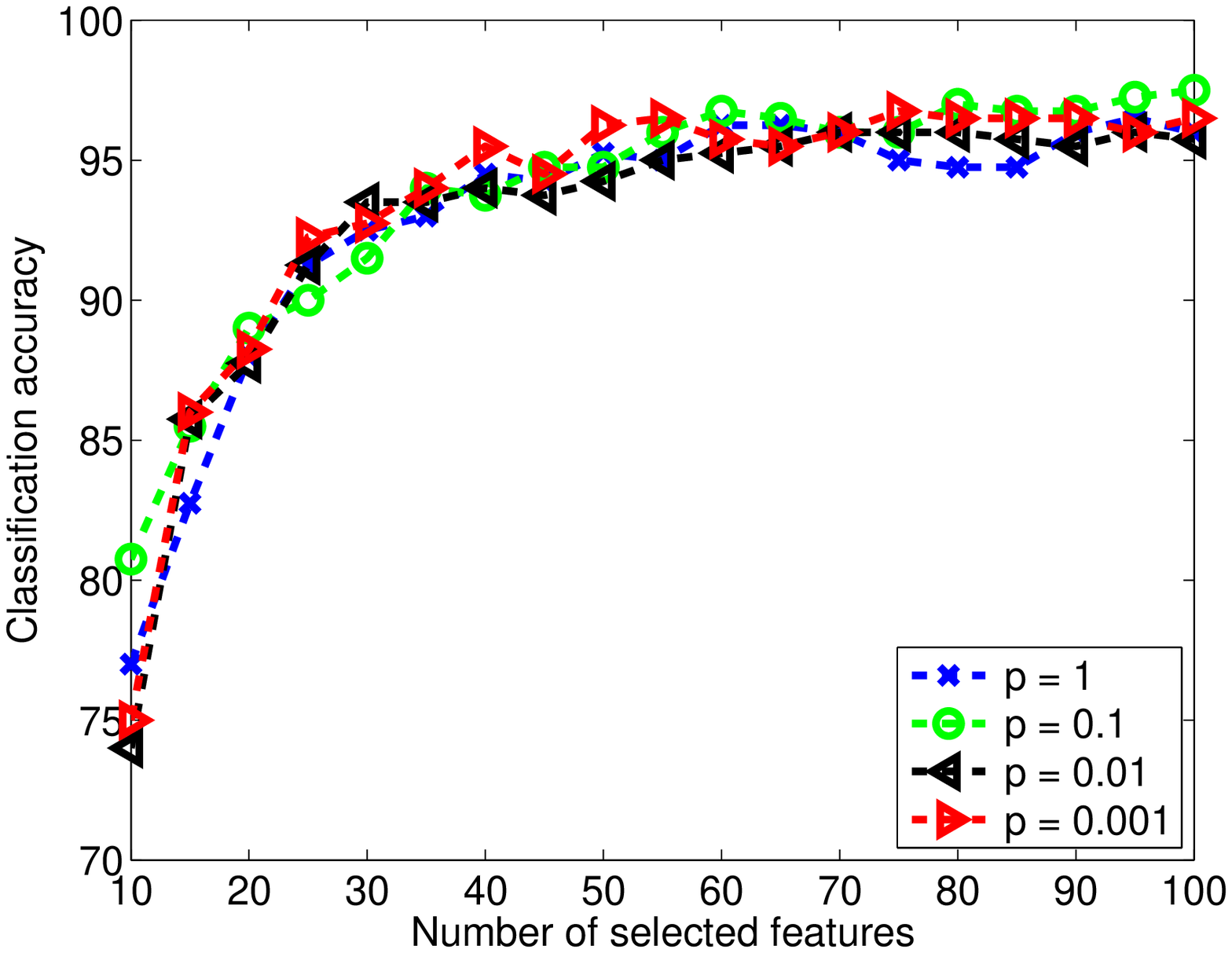}}
\subfigure[USPS]{
\includegraphics[width=0.3\textwidth]{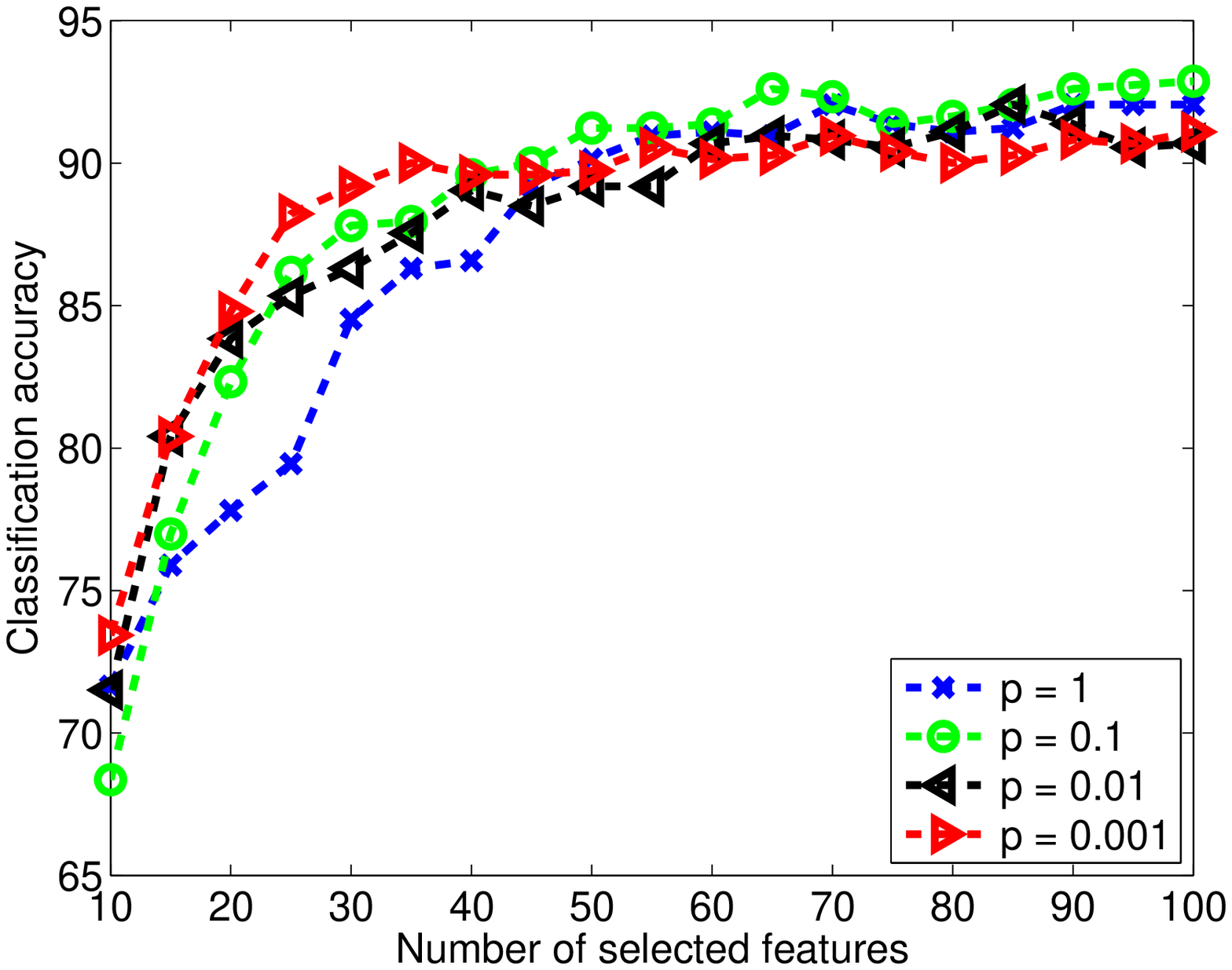}}
\subfigure[ISOLET5]{
\includegraphics[width=0.3\textwidth]{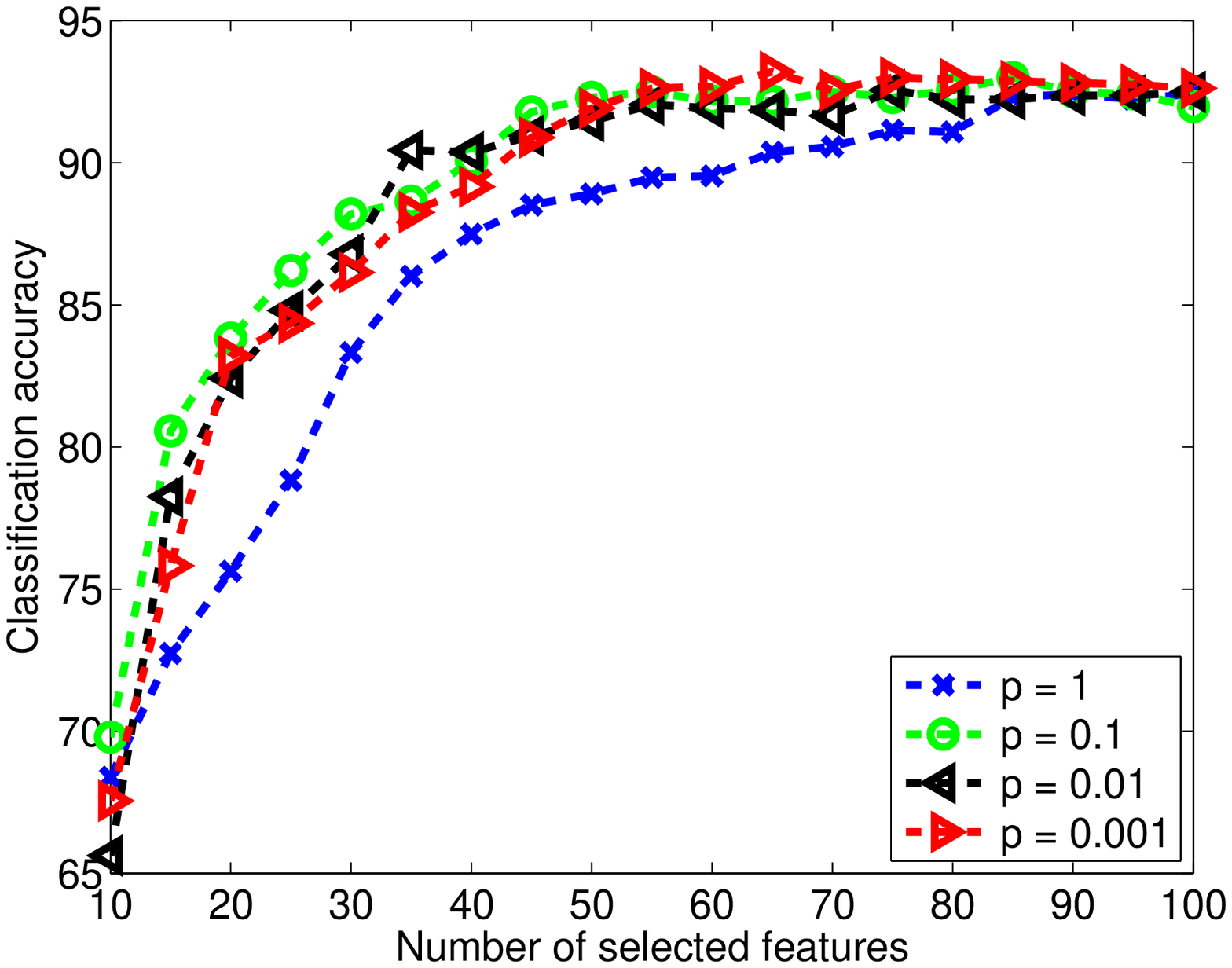}}
\caption{Comparison between DFS with different $p$ values w.r.t classification accuracy, $p = 0.001, 0.01, 0.1, 1$. }
\label{Fig.pEffect}
\end{figure*}

For brevity, we denote the performance of DFS with $p=c$ as DFS($p=c$).
From Fig. \ref{Fig.pEffect}, we can see that DFS with different $p$ values do lead to different results. On the data set ORL, the results of DFS for $p$ = 0.001, 0.01, 0.1 and 1, are very close. DFS($p=1$) is slightly behind the others when the selected features are more than 65. On the data set USPS, DFS($p=0.1$) outperforms DFS($p=1$) except when the number of selected features is 10. DFS($p=0.001$) and DFS($p=0.01$) show advantage over DFS($p=1$) when the selected features are fewer than 45. While on the ISOLET5 data set, when $p$ is less than 1, the performance of DFS consistently surpasses that of DFS($p=1$).

Though smaller $p$ value means sparser representation, the classification accuracy does not monotonically increase when $p$ decreases, as the results show. A possible reason is that our proposed algorithm only guarantees the local optimum for non-convex cases. Another reason may be that in the practical implementation of the algorithm, we may not manage to find the optimal value of the regularization parameter $\gamma$ for each $p$. On the other hand, in some cases, DFS with positive fractional $p$ values does find better solution than that when $p=1$. This is evidently demonstrated by the results on the data set ISOLET5.

\begin{figure*}
\centering
\subfigure[ORL, $p = 0.1$]{
\includegraphics[width=0.3\textwidth]{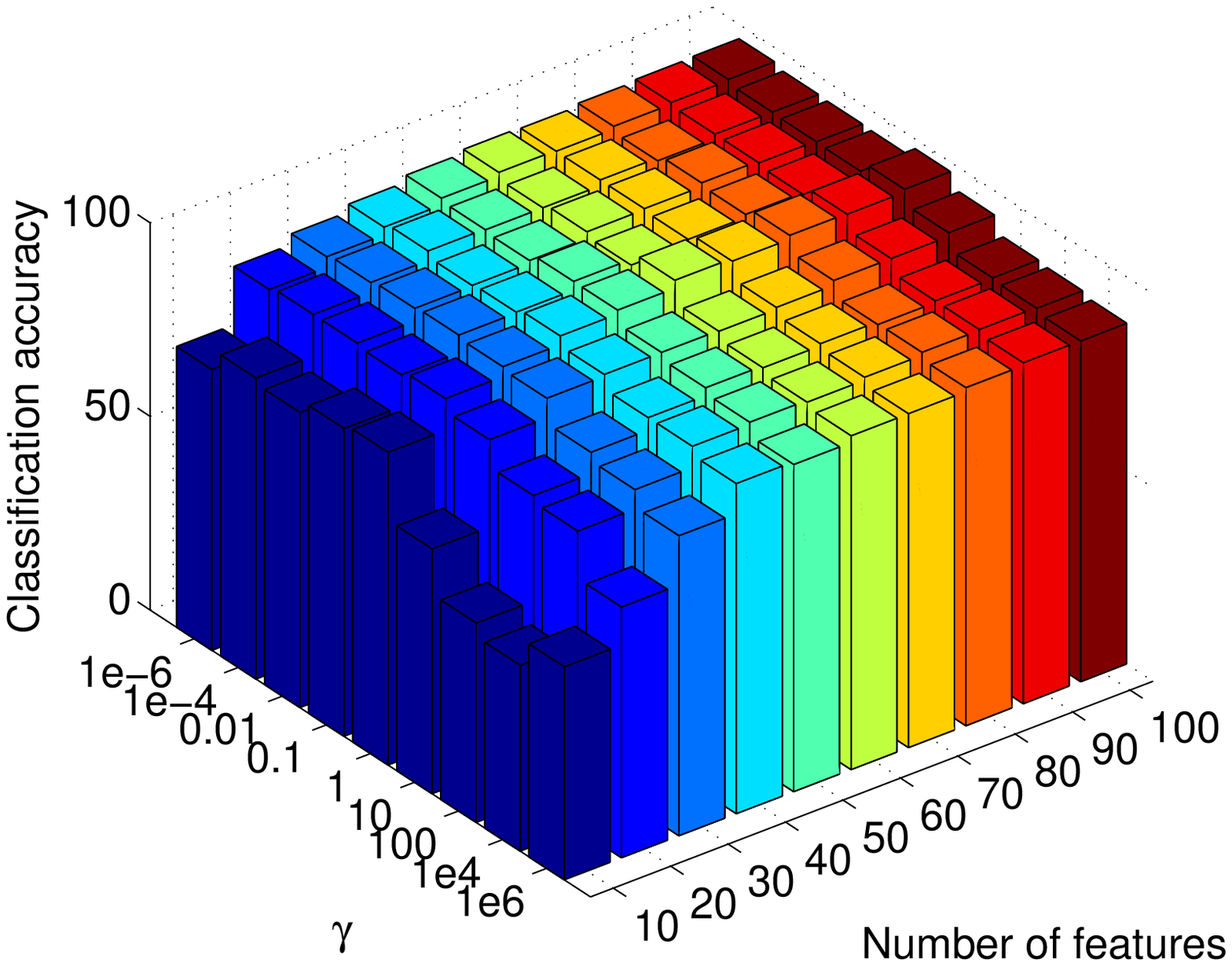}}
\subfigure[COLON, $p = 0.1$]{
\includegraphics[width=0.3\textwidth]{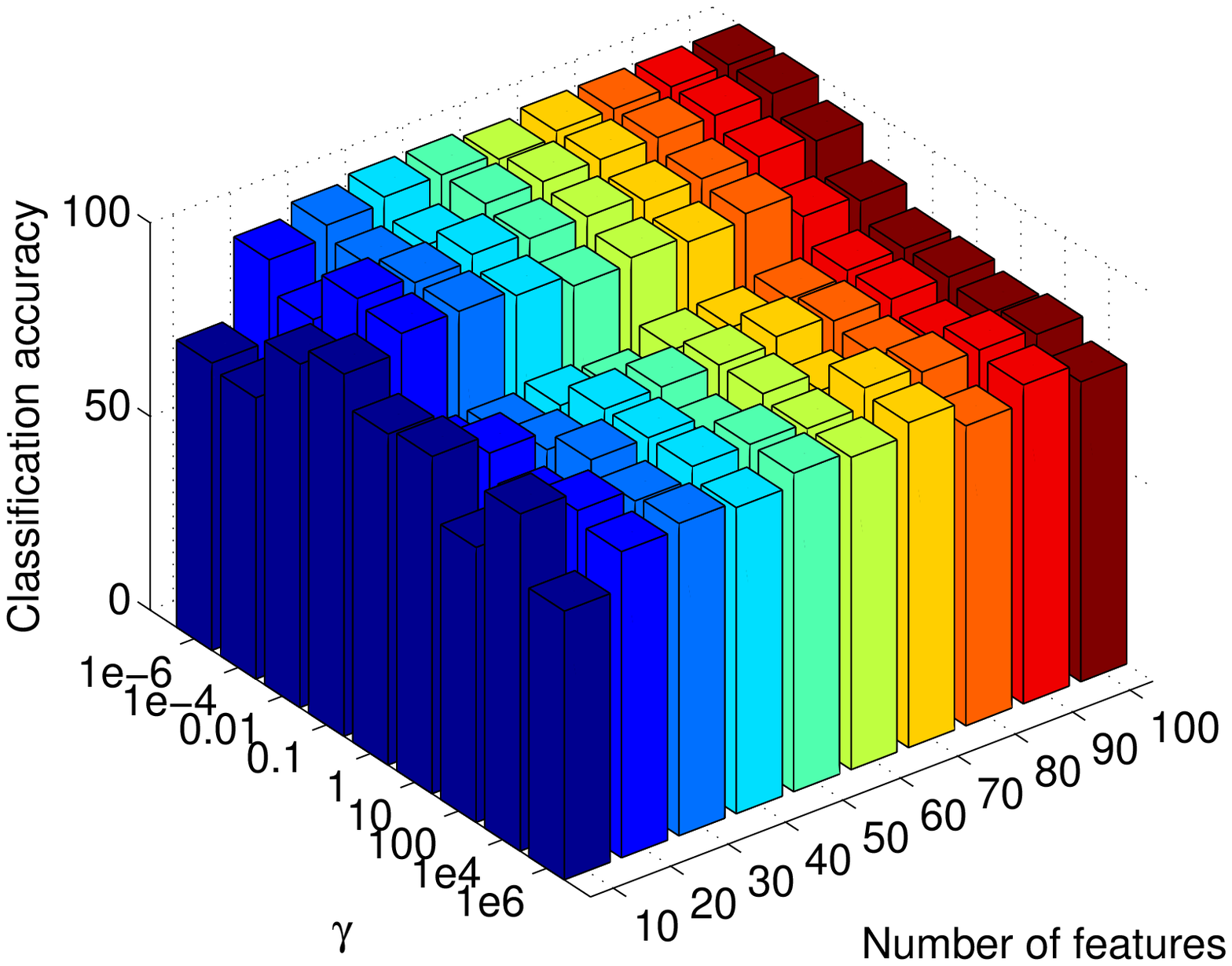}}
\subfigure[ISOLET5, $p = 0.1$]{
\includegraphics[width=0.3\textwidth]{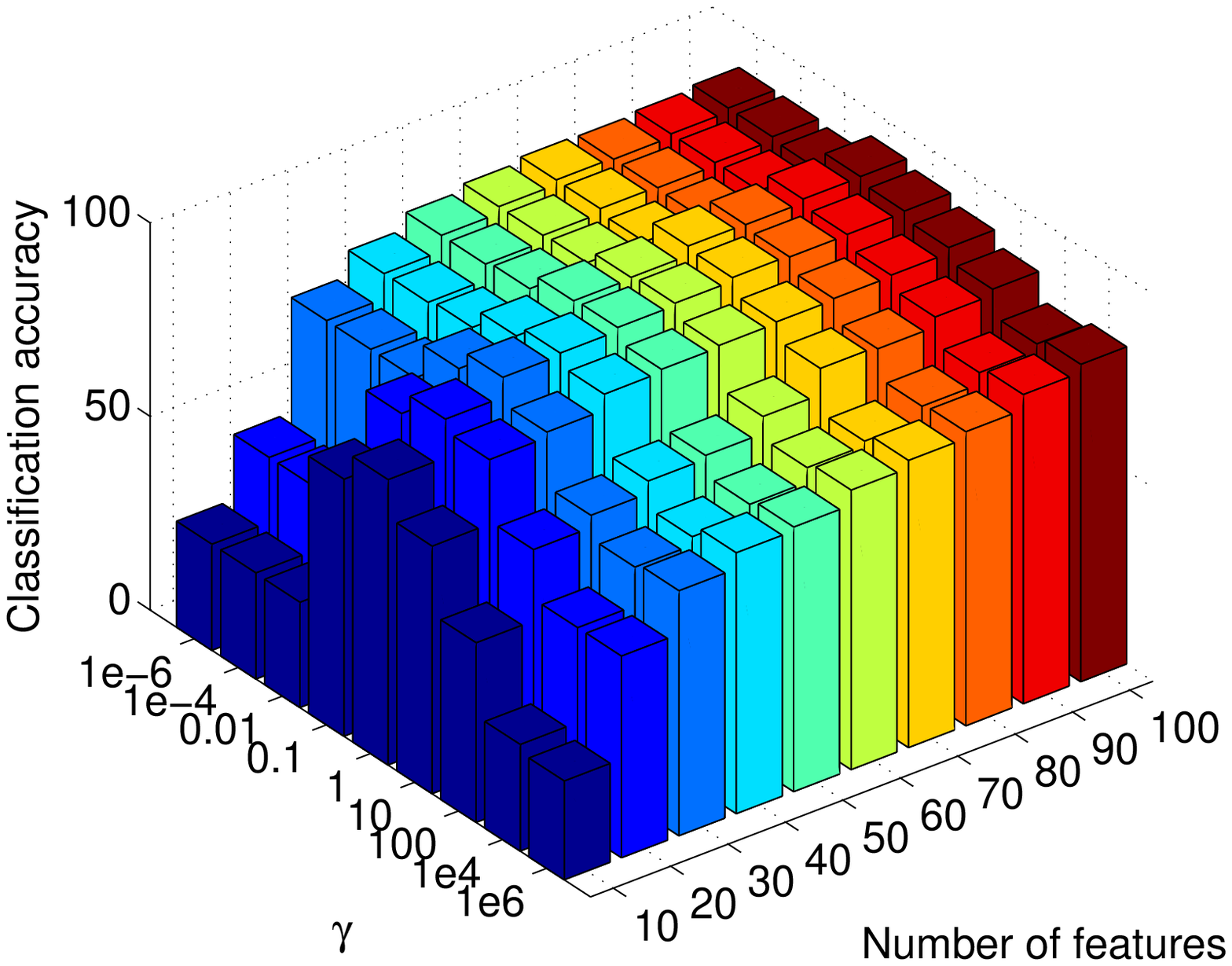}}

\subfigure[ORL, $p = 1$]{
\includegraphics[width=0.3\textwidth]{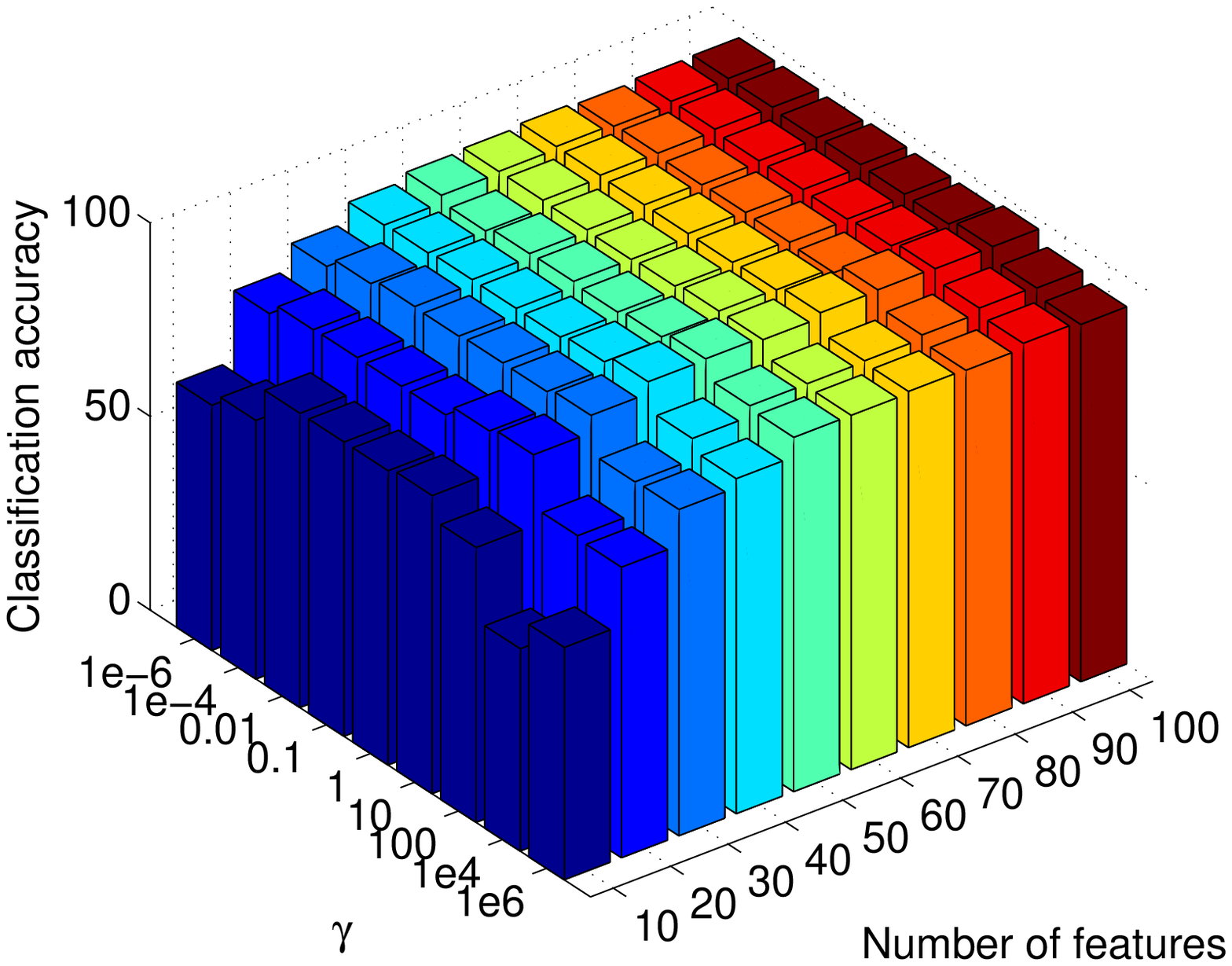}}
\subfigure[COLON, $p = 1$]{
\includegraphics[width=0.3\textwidth]{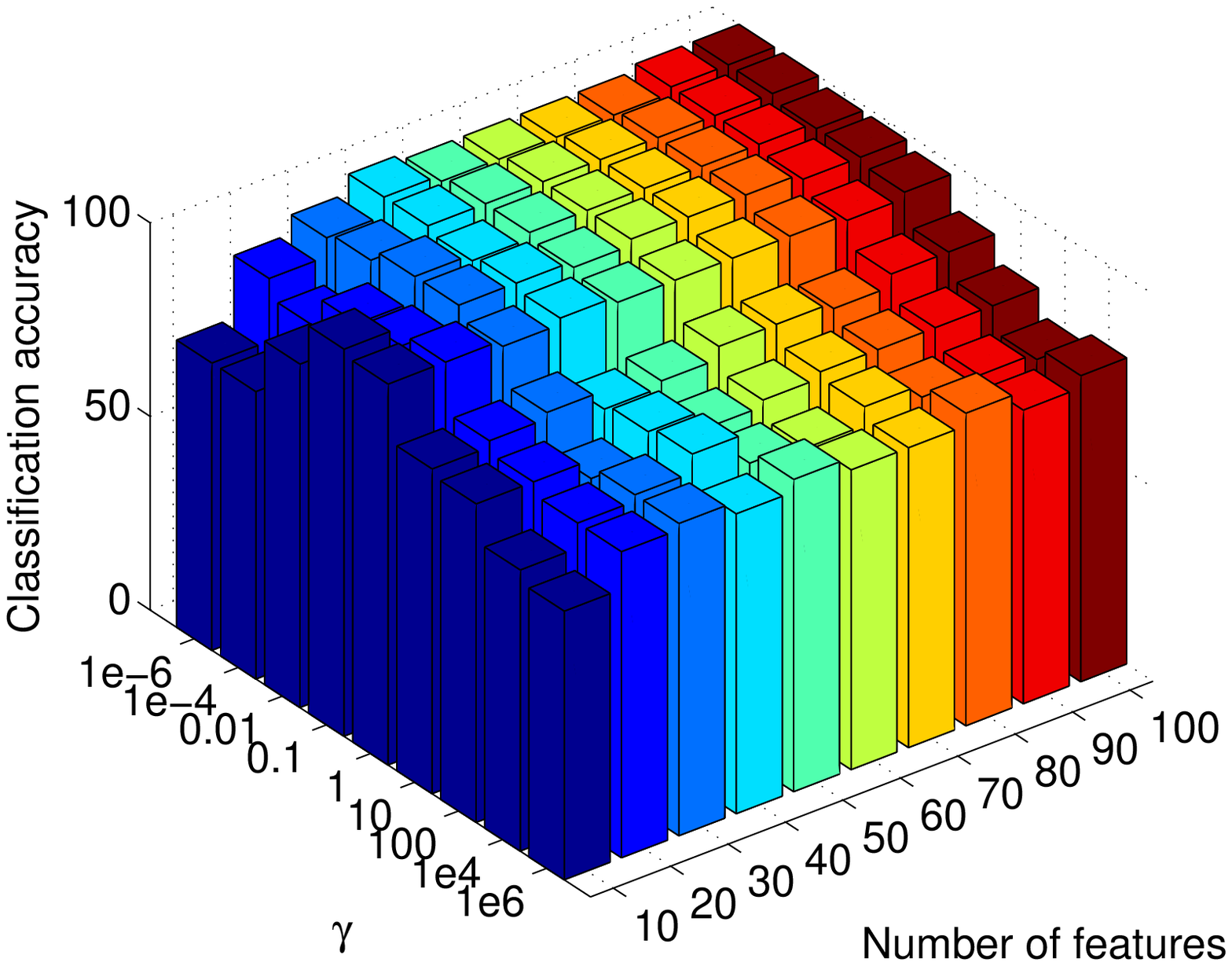}}
\subfigure[ISOLET5, $p = 1$]{
\includegraphics[width=0.3\textwidth]{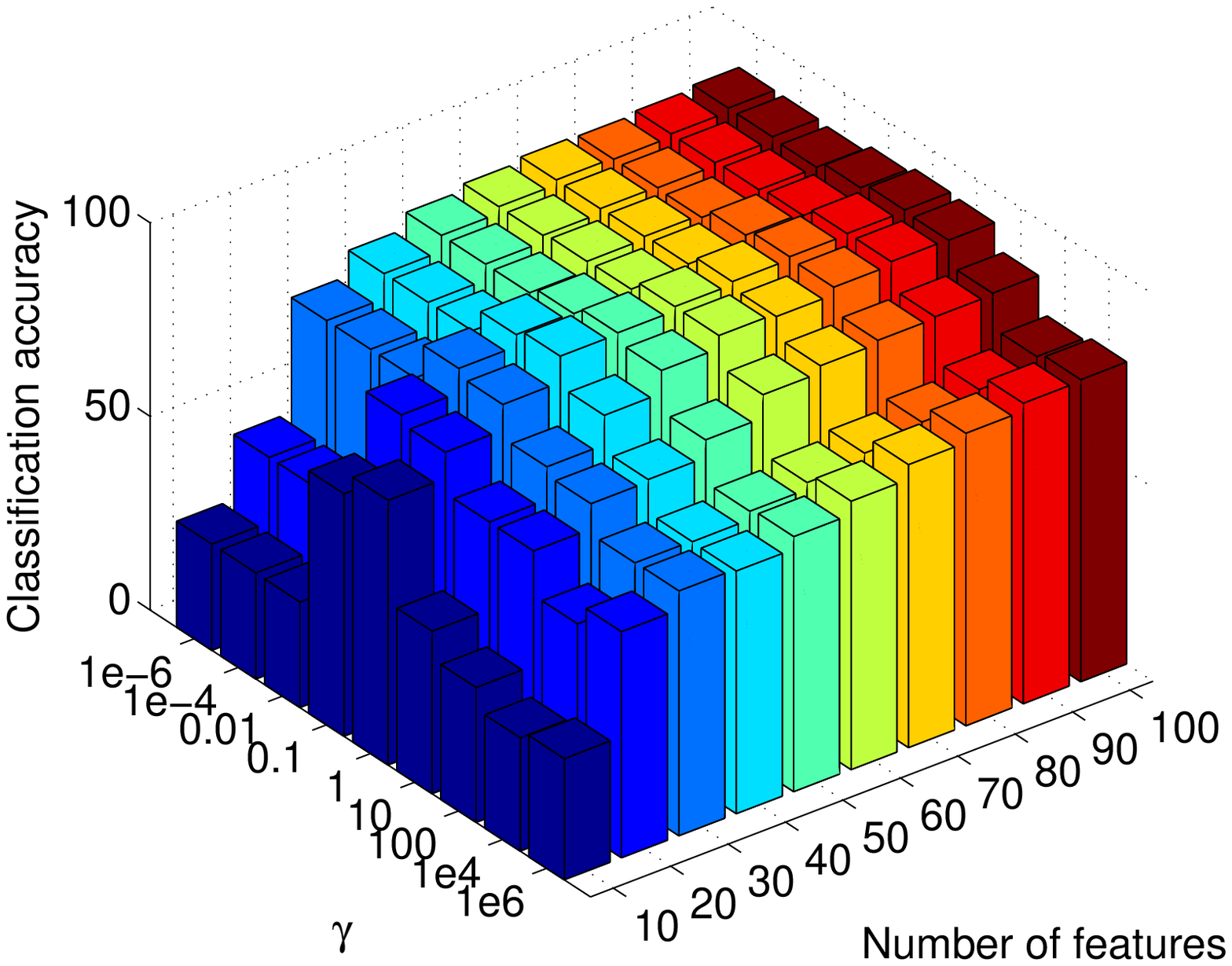}}
\caption{Performance variation of DFS with $p=0.1$ (top line) and $p=1$ (bottom line) w.r.t different values of the regularization parameter $\gamma$.}
\label{Fig.GammaEffect}
\end{figure*}

\subsection{Impact of $\gamma$ on The Performance of DFS}
The regularization parameter $\gamma $, which controls the trade-off between the criterion of LDA and the row sparsity of $\mathbf{A}$, plays an important role in DFS for feature selection. In this subsection, we study how the performance of DFS will be affected by different $\gamma$ values. Without loss of generality, we investigate the effect of $\gamma $ on DFS($p = 0.1$) and DFS($p = 1$) on data sets ORL, COLON and ISOLET5. The value of $\gamma $ is set as \{${{10}^{-6}}$, ${{10}^{-4}}$, 0.01, 0.1, 1, 10, 100, ${{10}^{4}}$, ${{10}^{6}}$\} and the number of selected features varies from 10 to 100 with an interval of 10. The performance variance w.r.t $\gamma $ and the number of selected features is showed in Fig. \ref{Fig.GammaEffect}.

As seen from Fig. \ref{Fig.GammaEffect}, DFS($p=1$) and DFS($p=0.1$) have similar performance variance trends w.r.t $\gamma$ on each data set, but different optimal $\gamma$ values.
The degree to which being affected of DFS by the value of $\gamma$ differs on these three data sets. From left to right, the performance variance w.r.t $\gamma$ grows. With the increasing of $\gamma$, the classification accuracy first ascend and then descend for both $p=0.1$ and $p=1$ on all data sets.
When the number of selected features is small, the performance of DFS is more sensitive to $\gamma$.
The performance variance created by varying $\gamma$ is comparable with that brought by different numbers of selected features.

\subsection{Convergence Analysis and Time Comparison}
To validate the efficiency of our proposed algorithm to solve DFS that involves ${\ell}_{2,p}$-norm minimization, we present the convergence behavior curves of Algorithm \ref{alg1:L2pLDFS} when $p=0.1,0.5,1$. Two kinds of results are provided. The first concerns the objective function and the other the divergence between two consecutive $\mathbf{A}$s, as shown in (\ref{DefiError}).
We show the results on data sets COIL20 and COLON since the algorithm has similar convergence behavior to the other data sets.
The convergence curves are displayed in Fig. \ref{Fig.COIL20Convergence}.

As seen from Fig. \ref{Fig.COIL20Convergence}, the objectives of DFS with $p = 0.1, 0.5, 1$ are non-increasing during the iterations, and they all converge to a fixed value. Additionally, in all cases, the divergence between two sequential $\mathbf{A}$s converges to zero, which indicates that the final results will not be changed drastically.
Furthermore, DFS converges within 20 iterations on this two data sets for the three $p$ values.
Therefore, our proposed DFS scales well in practice because of the fast convergence speed.

We report the computational time of DFS($p = 1$) and the other five baseline methods on two representative data sets COIL20 and ISOLET5. All the algorithms are tested on a laplop with 4 processors (2.27 GHz for each) and 5.87 GB available RAM memory by Matlab implementations\footnote{The code of BAHSIC offered on the author's website is written in Python. We have rewritten it in Matlab for fair comparisons.}. The results are shown in Table \ref{tabCalTime}.
As we have analyzed in Subsection B of Section V, eigen-decomposition is the most time consuming operation of DFS and it is performed in each iteration, thus DFS takes longer time. Similarly, BAHSIC involves iterations and needs to renew the data kernel matrix in each iteration, so it costs the most time in both cases.

\begin{table}[H]
    \centering
    \caption{Computational Time Comparison on Data Sets COIL20 and ISOLET5} \label{tabCalTime}
    \begin{tabular}{|l|c|c|c|c|c|c|}\hline
             & BAHSIC  &  LS  & mRMR &   RF  &   TR  &   DFS   \\ \hline
    COIL20   & 1135.93 & 0.41 & 7.25 & 51.42 & 1.23  & 96.56 \\ \hline
    ISOLET5  & 639.01  & 0.35 & 6.78 & 38.81 & 0.85  & 63.82   \\ \hline
\end{tabular}
\end{table}


\begin{figure*}
\centering
\subfigure[COIL20, $p=0.1$]{
\includegraphics[width=0.15\textwidth]{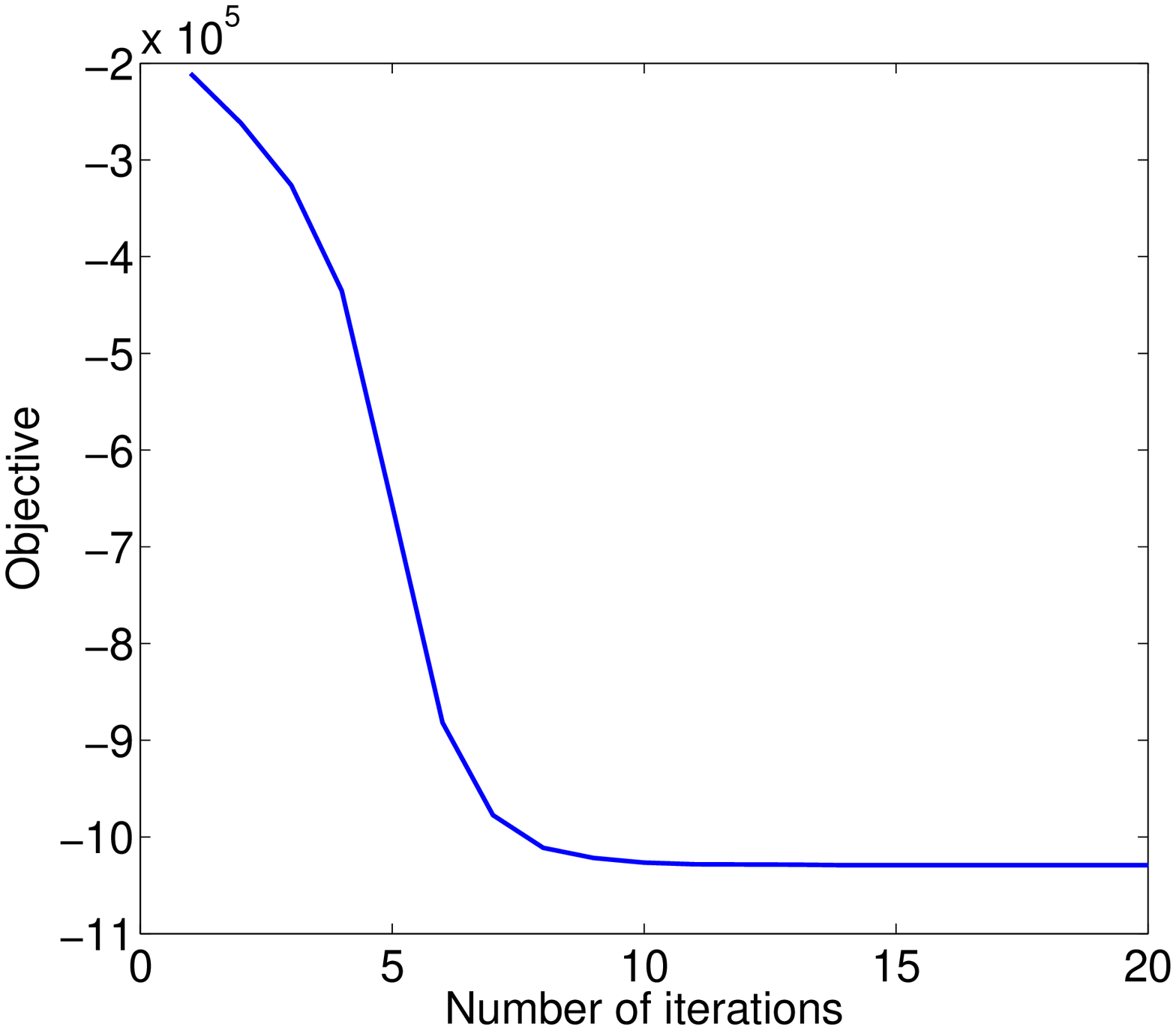}}
\subfigure[COIL20, $p=0.5$]{
\includegraphics[width=0.15\textwidth]{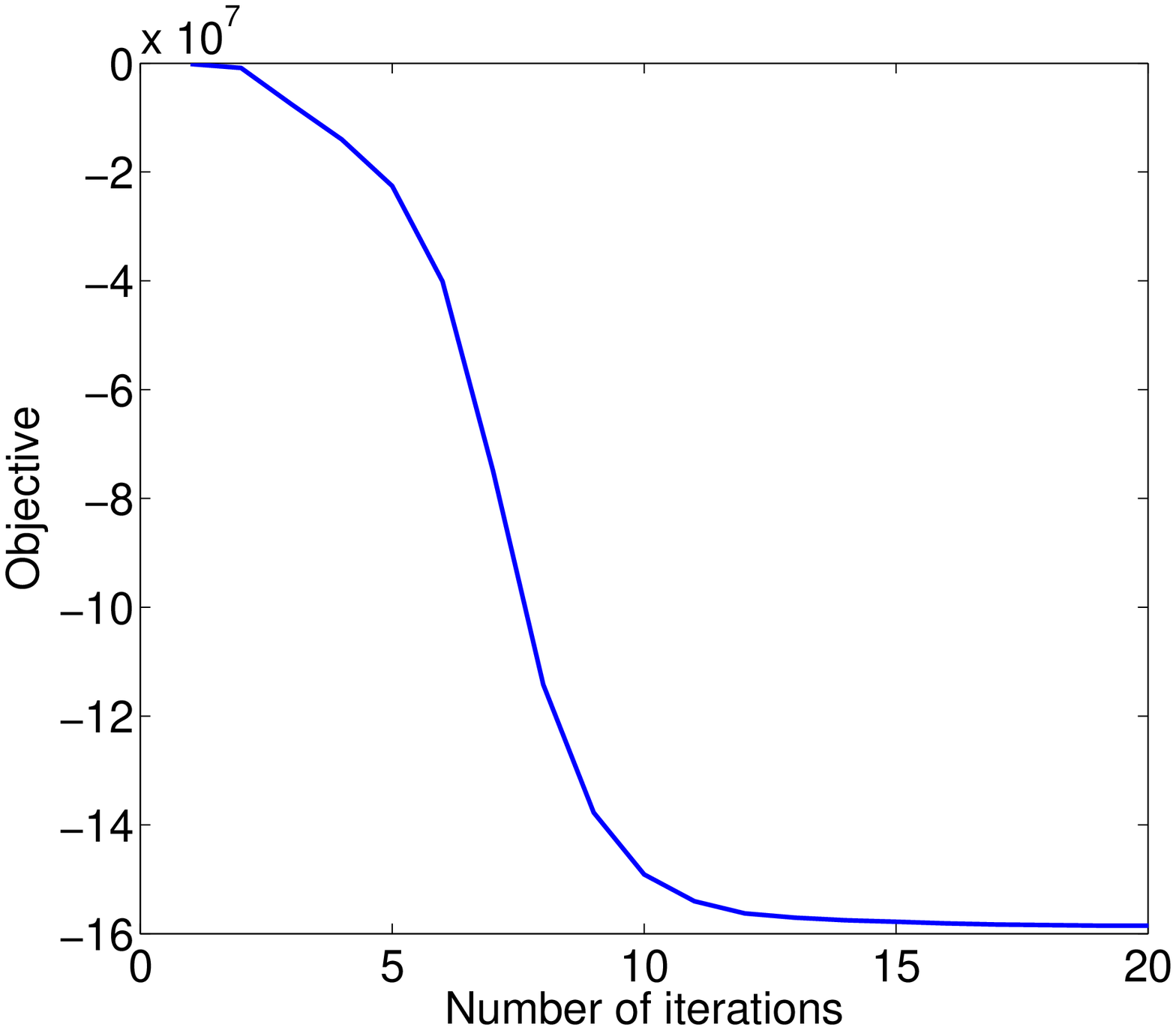}}
\subfigure[COIL20, $p=1$]{
\includegraphics[width=0.15\textwidth]{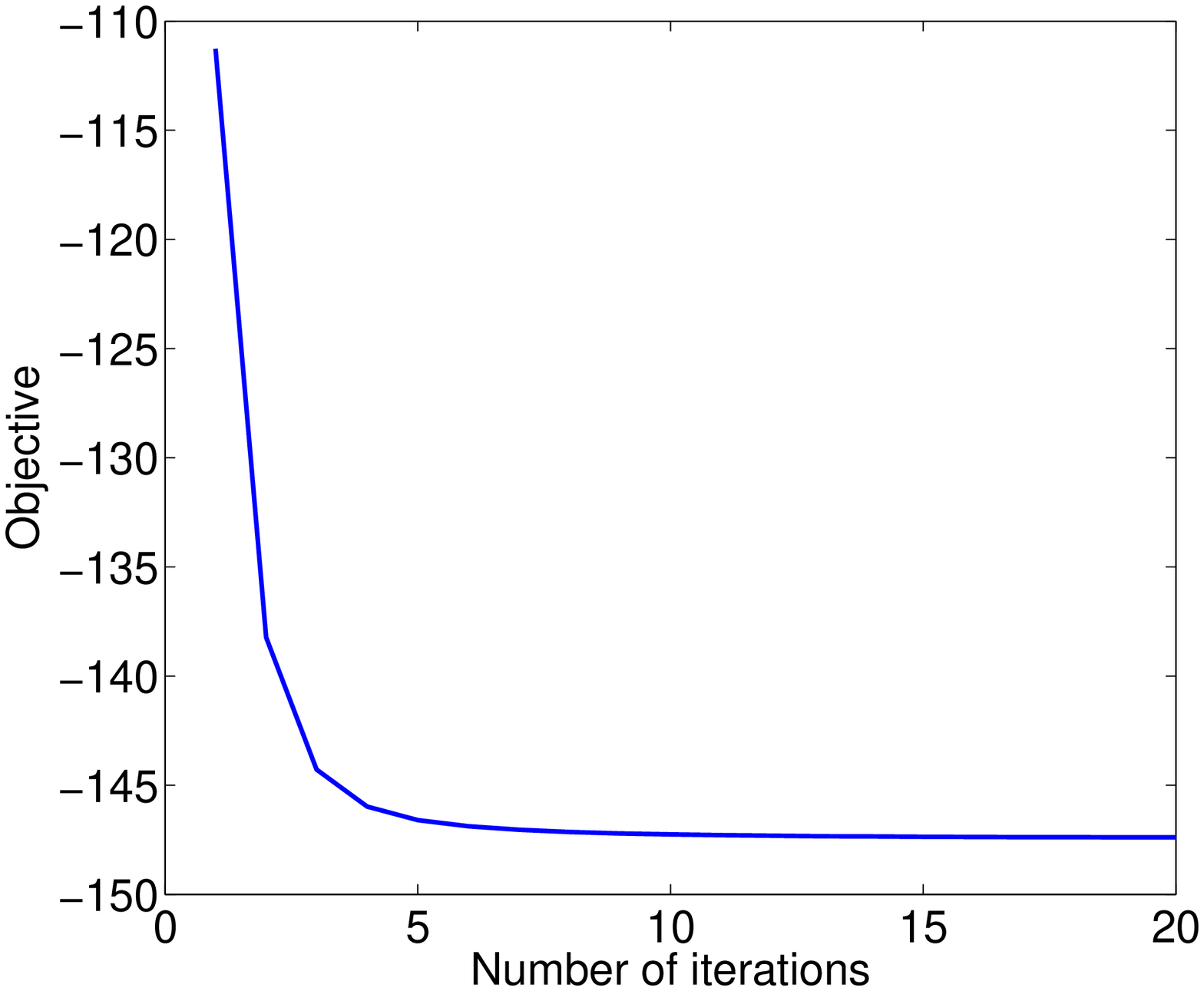}}
\subfigure[COLON, $p=0.1$]{
\includegraphics[width=0.15\textwidth]{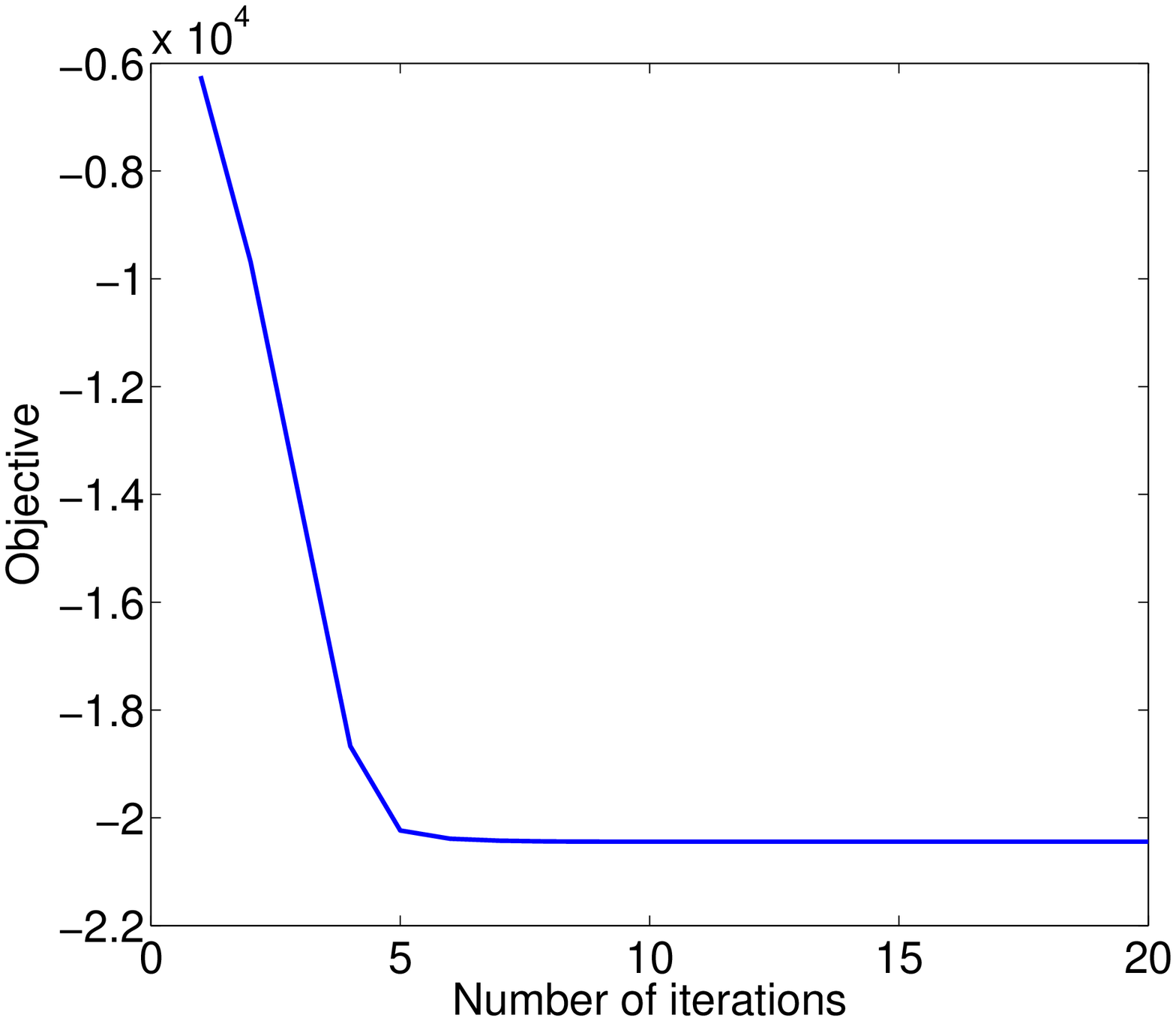}}
\subfigure[COLON, $p=0.5$]{
\includegraphics[width=0.15\textwidth]{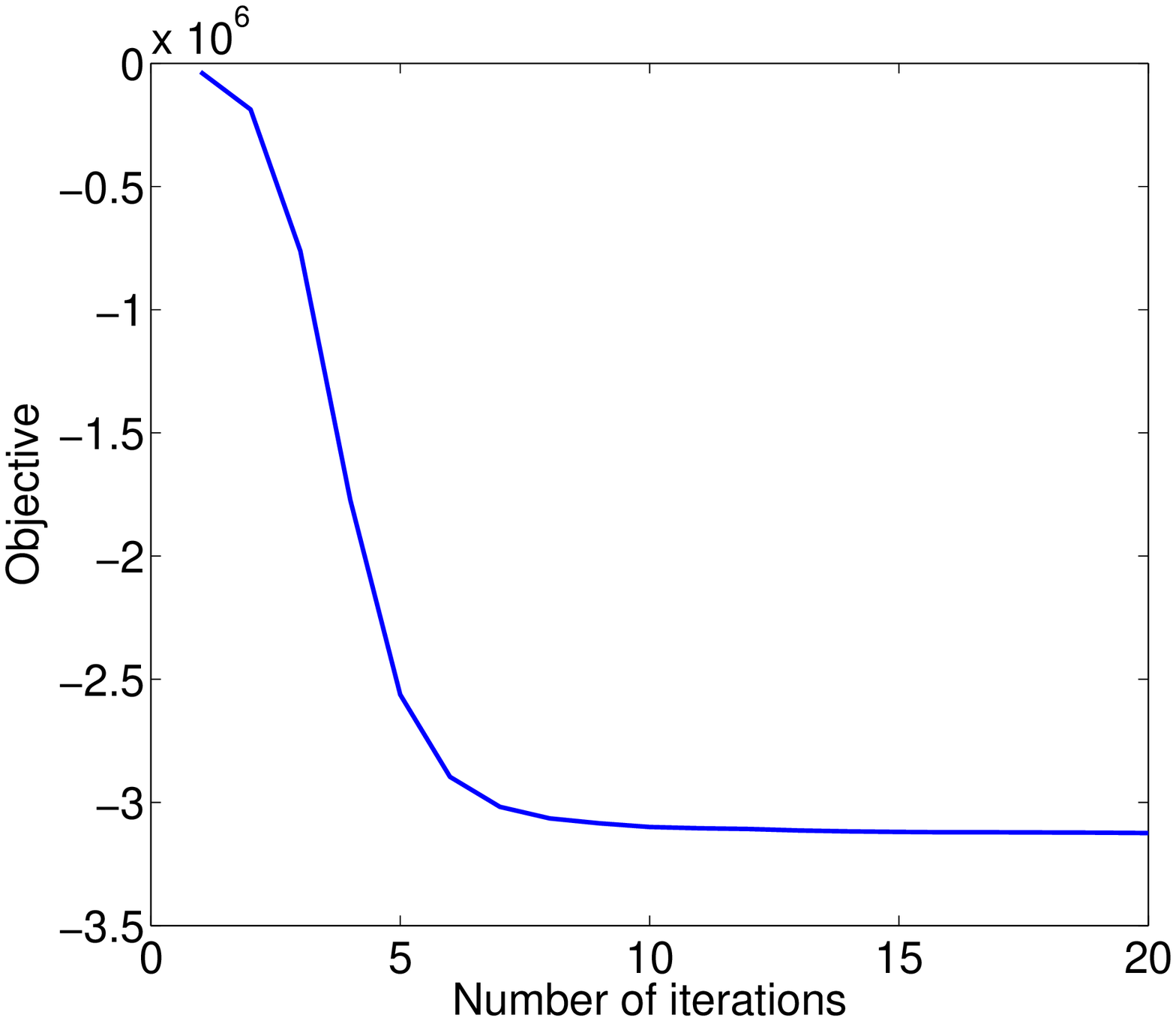}}
\subfigure[COLON, $p=1$]{
\includegraphics[width=0.15\textwidth]{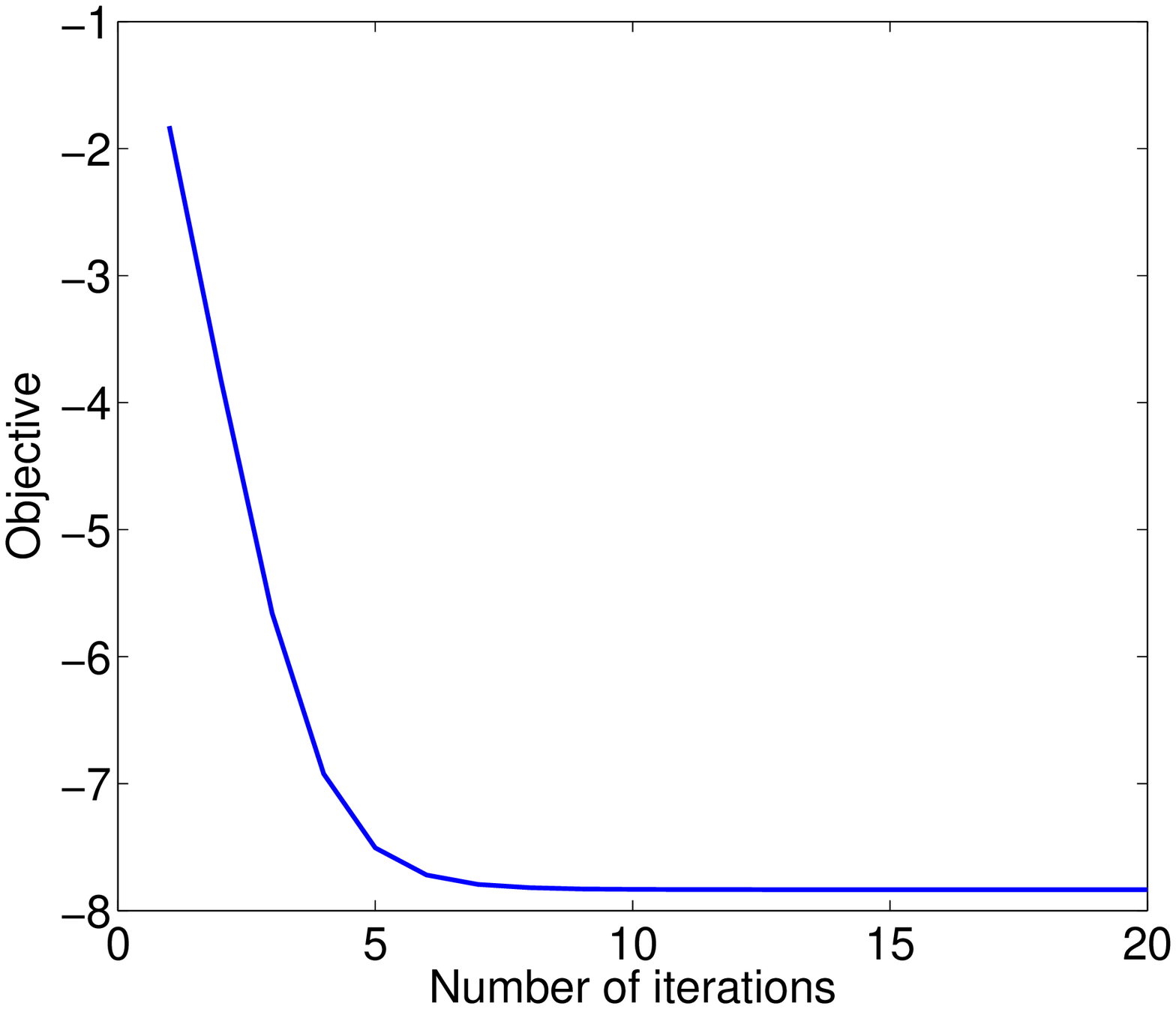}}

\subfigure[COIL20, $p=0.1$]{
\includegraphics[width=0.15\textwidth]{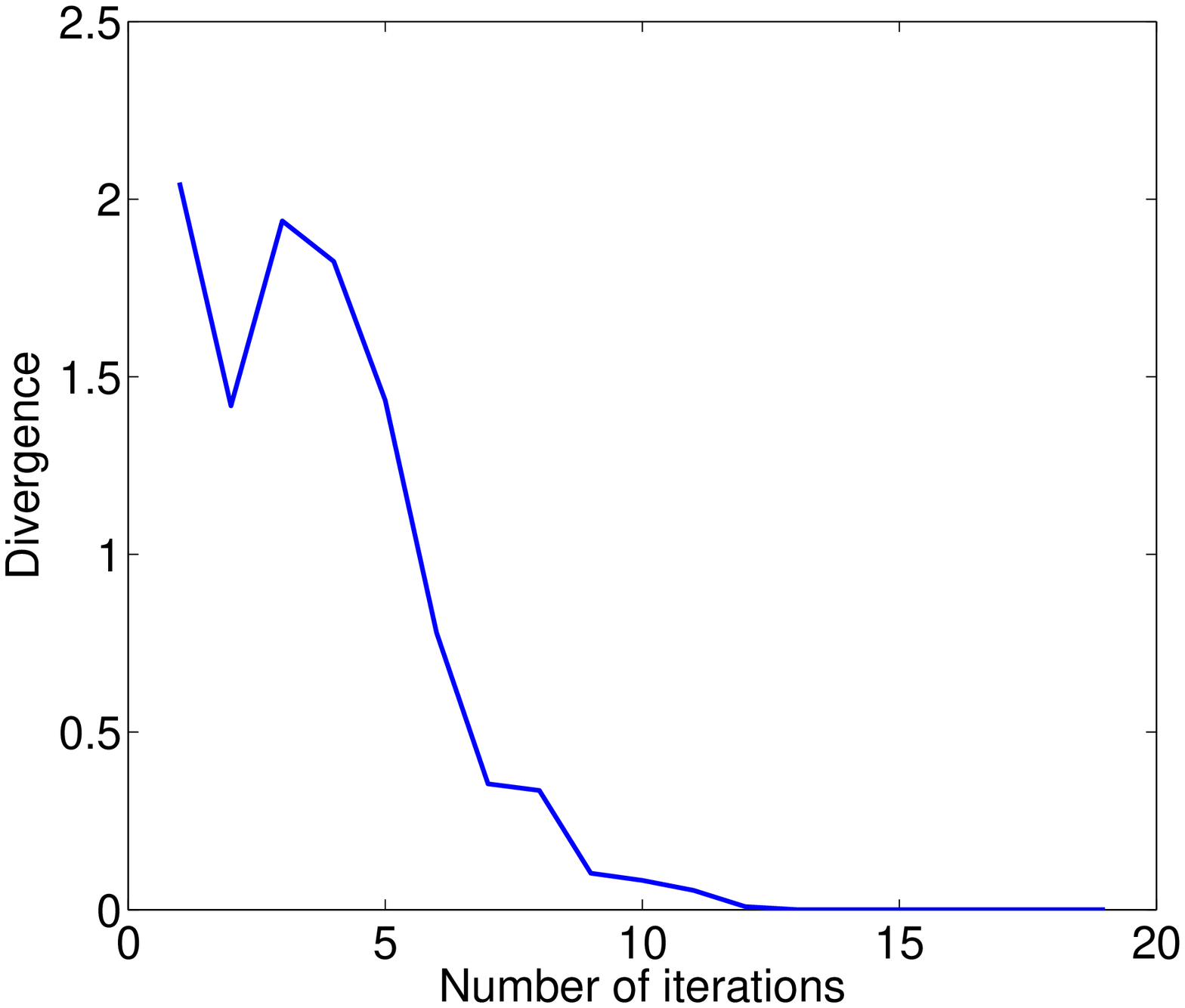}}
\subfigure[COIL20, $p=0.5$]{
\includegraphics[width=0.15\textwidth]{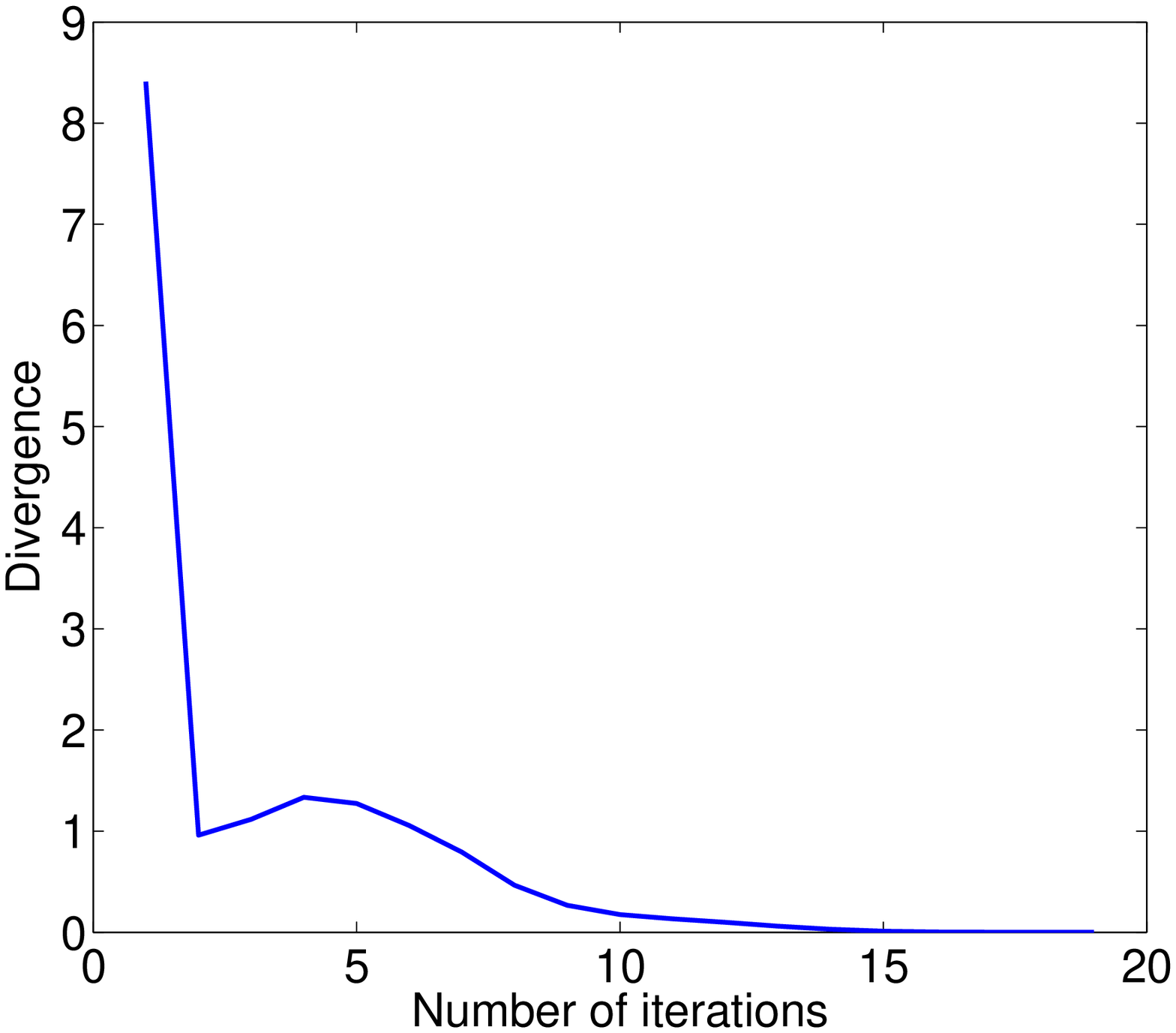}}
\subfigure[COIL20, $p=1$]{
\includegraphics[width=0.15\textwidth]{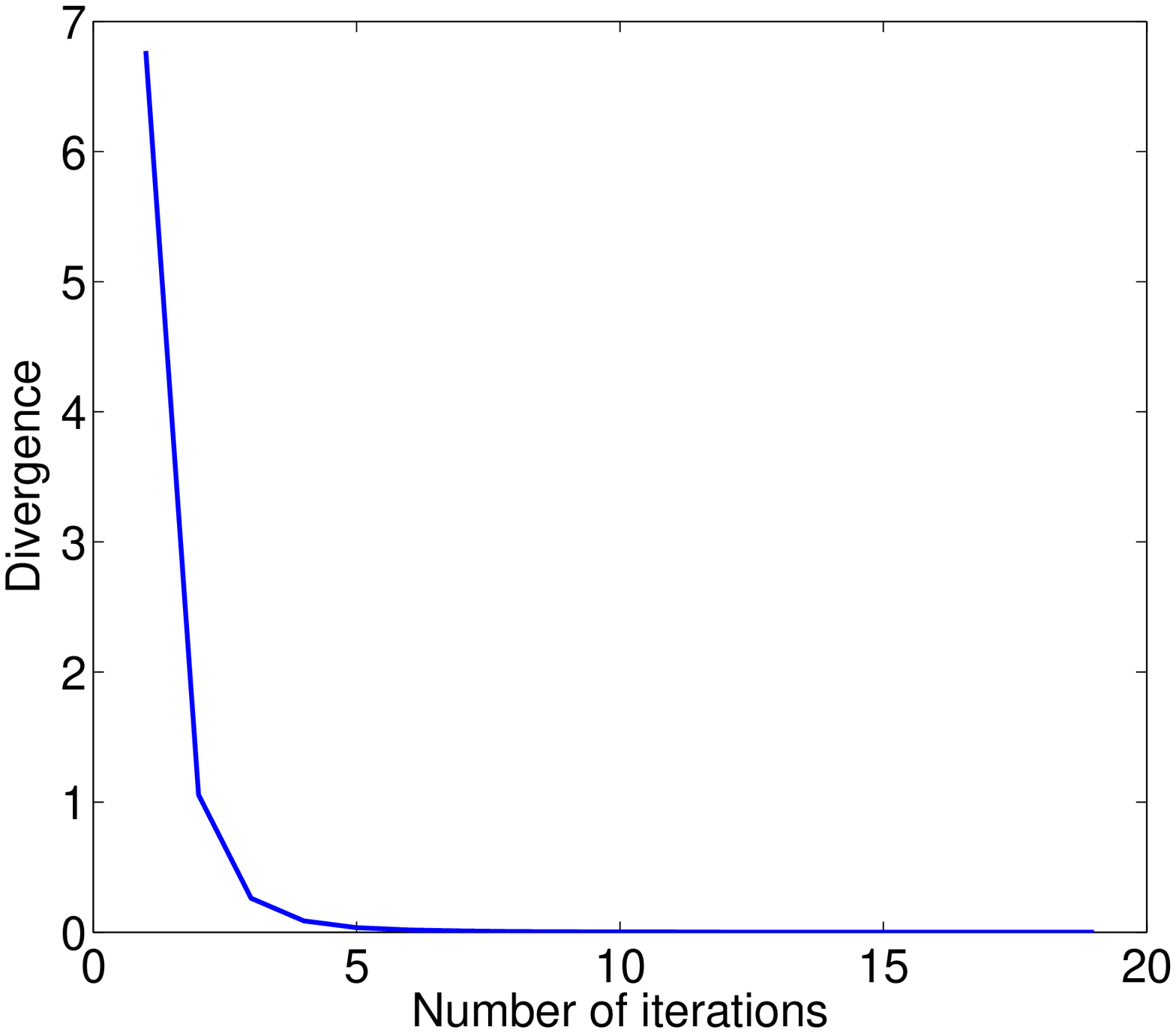}}
\subfigure[COLON, $p=0.1$]{
\includegraphics[width=0.15\textwidth]{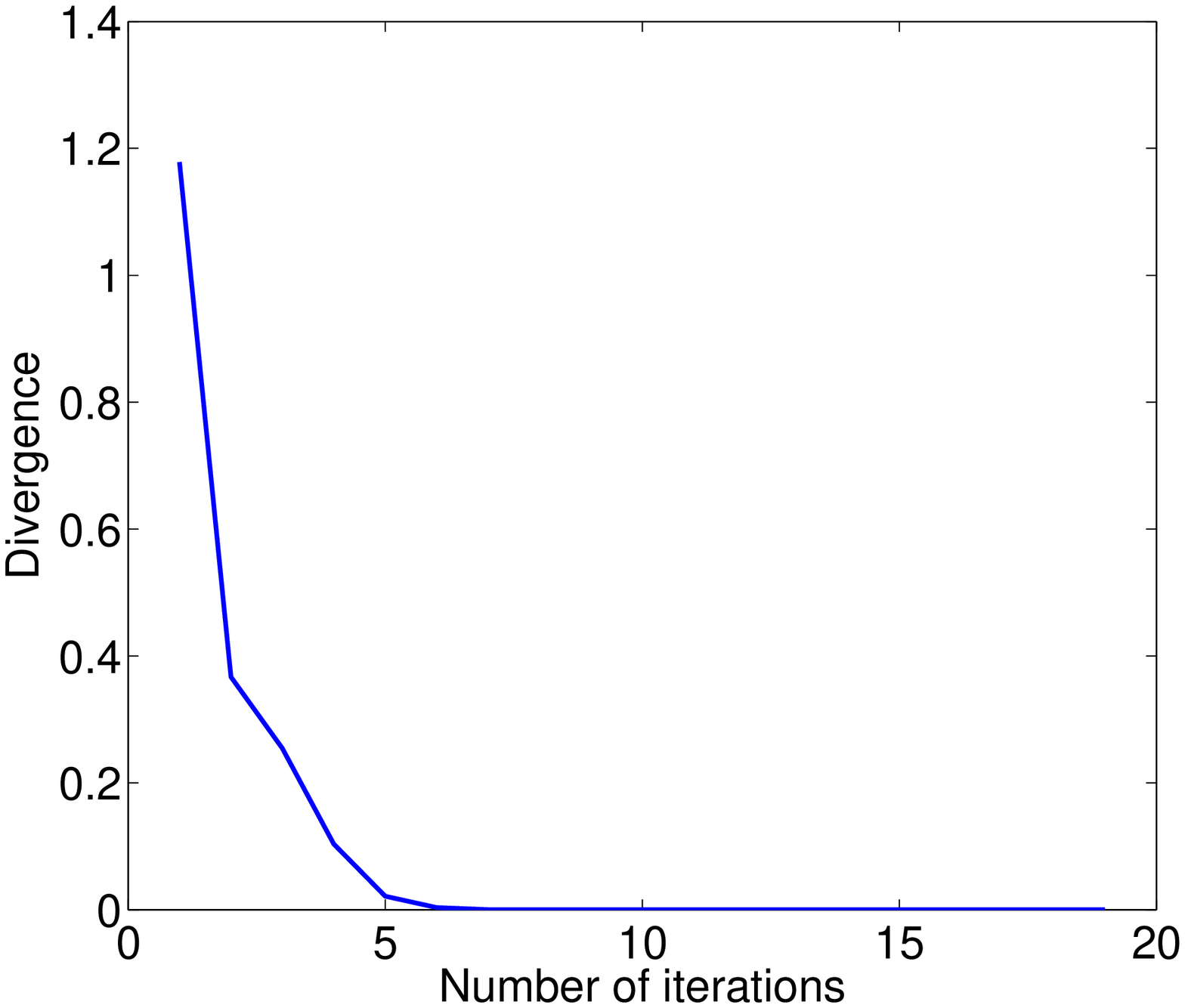}}
\subfigure[COLON, $p=0.5$]{
\includegraphics[width=0.15\textwidth]{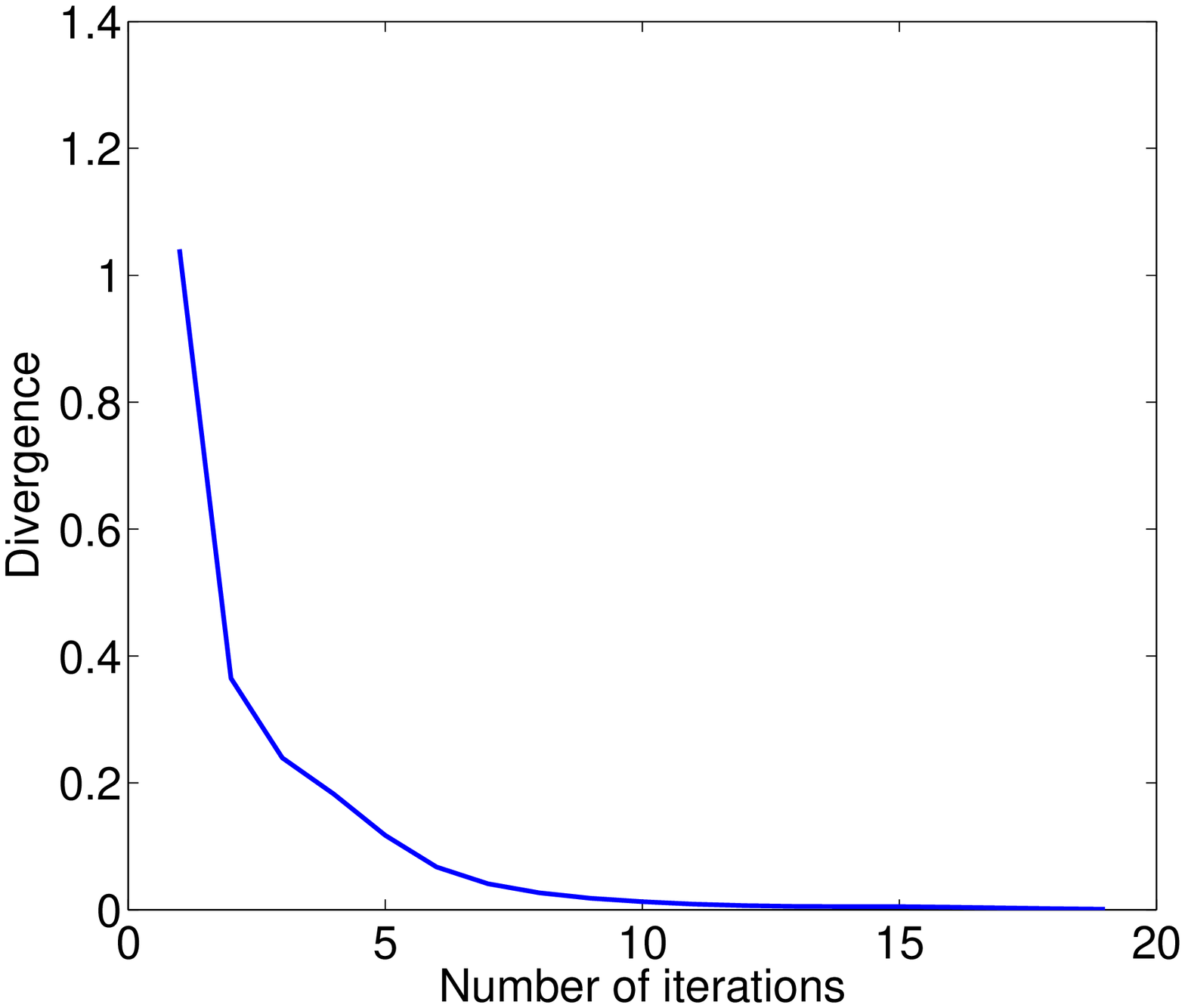}}
\subfigure[COLON, $p=1$]{
\includegraphics[width=0.15\textwidth]{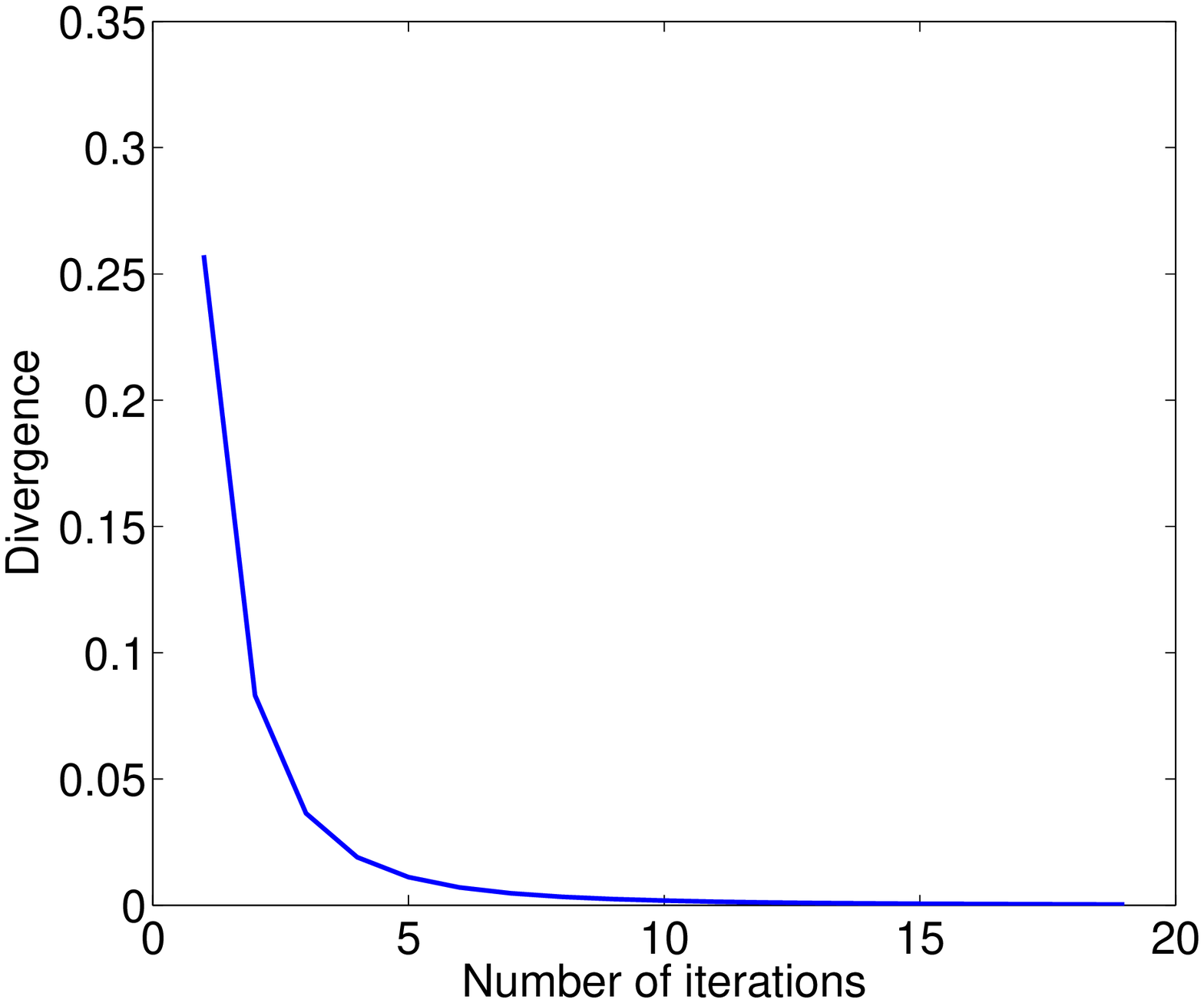}}

\caption{Convergence behavior of DFS on COIL20 (left side) and COLON (right side) respectively when $p=0.1, 0.5, 1$. Top line is the objective value of DFS. Bottom line is divergence between tow consecutive $\mathbf{A}$ measured by (\ref{DefiError}). }
\label{Fig.COIL20Convergence}
\end{figure*}

\section{Conclusion and Future Work}

In this paper, a new formulation is propounded by combining LDA and sparsity regularization for feature selection.
In particular, we manage to extend the ${\ell}_{2,1}$-norm based formulation to the ${\ell}_{2,p}$-norm regularized cases, providing more choices of $p$ to fit the variety of sparsity requirements. We derive an efficient algorithm to solve the ${\ell}_{2,p}$-norm minimization problem and prove that our algorithm will monotonically decrease the objective until convergence when $0<p\le 2$.
Moreover, our proposed DFS retains the ability to select the most discriminative features and remove the redundant ones simultaneously. This enables it to outperform competing feature selection methods. Experiments on various types of real-word data sets illustrate the advantages of our proposed method.

There are several interesting directions to investigate in the future. First, we would like to find a better way of dealing with the singularity of the total scatter matrix ${\mathbf{S}_{t}}$, which is addressed in this paper by regularization. Second, it is possible to extend this work to a kernel LDA version to deal with the nonlinear tasks. Finally, deciding the values of parameters is still an open problem, which is unsolved in many algorithms.

\bibliographystyle{IEEEtran}
\bibliography{DFSref}

\textbf{  }

\footnotesize{ \textbf{Hong TAO} is a PhD candidate with the College of Science at the National University of Defense Technology, Changsha, China. She earned her B.S. degree from the same university in 2012. Her research interests include machine learning, systems science and data mining.}

\footnotesize{ \textbf{Chenping HOU} (M'12) received his B.S. degree and the Ph.D. degree both in Applied Mathematics from the National University of Defense Technology, Changsha, China in 2004 and 2009, respectively. He is now a associate professor of College of Science in National University of Defense Technology. He has published several papers in the following journals and conferences: TNNLS/TNN, TSMCB, TIP, Pattern Recognition, IJCAI. He is also a member of IEEE and ACM. His research interests include pattern recognition, machine learning, data mining, and computer vision.}

\footnotesize{ \textbf{Feiping NIE} received the Ph.D. degree in Computer Science from Tsinghua University, China in 2009. His research interests are machine learning and its application fields, such as pattern recognition, data mining, computer vision, image processing and information retrieval. He has published more than 100 papers in the following top journals and conferences: TPAMI, IJCV, TIP, TNNLS/TNN, TKDE, TKDD, TVCG, TCSVT, TMM, TSMCB/TC, Machine Learning, Pattern Recognition, Medical Image Analysis, Bioinformatics, ICML, NIPS, KDD, IJCAI, AAAI, ICCV, CVPR, SIGIR, ACM MM, ICDE, ECML/PKDD, ICDM, MICCAI, IPMI, RECOMB. According to the Google scholar, his papers have been cited more than 2000 times. He is now serving as Associate Editor or PC member for
several prestigious journals and conferences in the related fields.}

\footnotesize{ \textbf{Yuanyuan JIAO} received the Ph.D. degree in Computer Science from in Applied Mathematics from the National University of Defense Technology, Changsha, China in 2012. Her research interests include data mining and its applications.}

\footnotesize{ \textbf{Dongyun YI} is Professor with the College of Science at the National University of Defense Technology, Changsha, China. He earned his B.S. degree from Nankai University, Tianjin, China and the M.S. and Ph.D. degrees from National University of Defense Technology in Changsha, China, respectively. He has worked as a visiting researcher at the University of Warwick in 2008. His research interests include statistics, systems science and data mining.}

\end{document}